\documentclass[11pt]{article}

\usepackage[utf8]{inputenc} 
\usepackage[T1]{fontenc}    

\usepackage{hyperref}       
\usepackage{url}            
\usepackage{booktabs}       
\usepackage{amsfonts}       
\usepackage{nicefrac}       
\usepackage{microtype}      
\usepackage{xcolor}         

\usepackage[square,numbers]{natbib}
\usepackage{graphicx}
\usepackage{latexsym}
\usepackage{epsfig}
\usepackage{amssymb}
\usepackage{amstext}
\usepackage{amsgen}
\usepackage{amsxtra}
\usepackage{amsgen}
\usepackage{amsthm,mathrsfs,mathtools}
\usepackage{bbm}
\usepackage{enumerate,amsmath,subfigure,color}
\usepackage{tikz}

\usepackage[verbose=true,letterpaper]{geometry}
\AtBeginDocument{
  \newgeometry{
    textheight=9in,
    textwidth=6.35in,
    top=1in,
    headheight=12pt,
    headsep=25pt,
    footskip=30pt
  }
}

\newtheorem{thm}{Theorem}[section]
\newtheorem{proposition}[thm]{Proposition}
\newtheorem{lemma}[thm]{Lemma}

\theoremstyle{remark}
\newtheorem{rem}[thm]{Remark}

\theoremstyle{definition}
\newtheorem{example}[thm]{Example}

\newtheorem{assumption}[thm]{Assumption}

\newcommand{\N}{\mathbb{N}}
\newcommand{\Z}{\mathbb{Z}}
\newcommand{\R}{\mathbb{R}}

\newcommand{\T}{\mathbb{T}}
\newcommand{\E}{\mathbb{E}}
\newcommand{\Prob}{\mathbb{P}}

\newcommand{\upchi}{{\text{\raisebox{2pt}{$\chi$}}}}
\renewcommand{\epsilon}{\varepsilon} 
\renewcommand{\Im}{\operatorname{Im}}

\DeclareMathOperator*{\argmin}{arg\,min}
\newcommand{\tr}{\operatorname{tr}} 

\def\<{\langle}
\def\>{\rangle}
\newcommand{\lY}{\langle}
\newcommand{\rY}{\rangle_{Y}}
\newcommand{\lK}{\langle}
\newcommand{\rK}{\rangle_{K^* \times K}}

\newcommand{\Se}{\Sigma_{\varepsilon}}
\newcommand{\Sx}{\Sigma_x}

\newcommand{\vH}{H} 
\newcommand{\vh}{h} 
\newcommand{\vz}{z} 
\newcommand{\vzv}{{\mathbf{\vz}}} 

\newcommand{\hOp}{\vh^{\star}}
\newcommand{\BOp}{B^\star}
\newcommand{\thetaOp}{\theta^\star}
\newcommand{\xOp}{x^\star}
\newcommand{\WOp}{W^\star}
\newcommand{\hU}{\widehat{\vh}_U}
\newcommand{\BU}{\widehat{B}_U}
\newcommand{\thetaU}{\widehat{\theta}_U}
\newcommand{\xU}{\widehat{x}_U}
\newcommand{\hS}{\widehat{\vh}_S}
\newcommand{\BS}{\widehat{B}_S}
\newcommand{\thetaS}{\widehat{\theta}_S}

\newcommand{\wh}[1]{\widehat{#1}}
\newcommand{\wt}[1]{\widetilde{#1}}

\newcommand{\discr}[1]{\boldsymbol{#1}}
\newcommand{\spec}[1]{\bar{#1}}
\newcommand{\specc}[1]{\bar{\bar{#1}}}

\newcommand{\hh}{H}

\newcommand{\eps}{\varepsilon}

\title{Learning the optimal Tikhonov regularizer for inverse problems}

\author{\hspace{-0.75cm}Giovanni S.~Alberti$^1$, Ernesto De Vito$^1$, Matti Lassas$^2$, Luca Ratti$^1$, Matteo Santacesaria$^1$}

\date{\small
$^1$ MaLGa Center, Department of Mathematics, University of Genoa, \\ Via Dodecaneso 35,  16146 Genova,  Italy \\ \texttt{giovanni.alberti},\texttt{ernesto.devito},\texttt{luca.ratti},\texttt{matteo.santacesaria@unige.it} \\
$^2$ Department of Mathematics and Statistics, University of Helsinki,\\ Gustaf H\"allstr\"omin katu 2,  
00014 Helsinki, Finland \\\texttt{matti.lassas@helsinki.fi}
}
\begin{document}

\maketitle

\begin{abstract}

In this work, we consider the linear inverse problem $y=Ax+\epsilon$, where $A\colon X\to Y$ is a known linear operator between the separable Hilbert spaces $X$ and $Y$, $x$ is a random variable in $X$ and $\epsilon$ is a zero-mean random process in $Y$. This setting covers several inverse problems in imaging including denoising, deblurring and X-ray tomography. Within the classical framework of regularization, we focus on the case where the regularization functional is not given a priori, but learned from data. Our first result is a characterization of the optimal generalized Tikhonov regularizer, with respect to the mean squared error. We find that it is completely independent of the forward operator $A$ and depends only on the mean and covariance of $x$. 
Then, we consider the problem of learning the regularizer from a finite training set in two different frameworks: one supervised, based on samples of both $x$ and $y$, and one unsupervised, based only on samples of $x$. In both cases we prove generalization bounds, under some weak assumptions on the distribution of $x$ and $\eps$, including the case of sub-Gaussian variables. Our bounds hold in infinite-dimensional spaces, thereby showing that finer and finer discretizations do not make this learning problem harder. The results are validated through numerical simulations.

\end{abstract}

\section{Introduction}
\label{sec:intro}
The aim of an inverse problem is to recover information about a physical quantity from indirect measurements. Virtually all imaging problems and  modalities fall within this framework, including denoising, deblurring \cite{engl1996}, computed tomography \cite{natterer2001mathematics} and magnetic resonance imaging \cite{epstein2007introduction}. Classical and general approaches to solve inverse problems consist in studying a variational (minimization) problem and can be divided into two classes. 

The first is based on the so-called regularization theory \cite{engl1996}. The aim is to recover a single, deterministic, unknown $x^\dag$ from noisy data $y = F(x^\dag) +\varepsilon$ by solving a minimization problem 
\begin{equation}\label{eq:varmin}
\min_{x} d_Y(F(x),y) + J(x)
\end{equation}
for a fidelity term $d_Y\colon Y \times Y \to \R $ and a regularization functional $J\colon X \to [0,+\infty)$. The latter is chosen in order to mitigate the ill-posedness of the map $F$, and  represent some a-priori knowledge on $x$. For instance, in the classical Tikhonov regularization we have $J(x) = \lambda \|x\|_X^2$ for $\lambda>0$.

The second approach considers the unknown as a random variable and is based on statistical/Bayesian methods \cite{kaipio2006,stuart-2010}. In this case one can recover the unknown using point estimators such as the maximum a posteriori (MAP) estimator, or extract richer information on the probability distribution of the unknown. In practice, the MAP estimator is found by solving a minimization problem of the same form as \eqref{eq:varmin}. The main difference is that the fidelity term and the regularizer are tailored to the statistical properties of the unknown and the noise, which are usually assumed to be known.

In recent years, machine learning techniques, and especially deep learning, have shaken the field of inverse problems by providing the basis for data-driven methods that have outperformed the state-of-the-art in most imaging modalities  \cite{arridge2019,2020-ongie-etal}. The most successful methods take inspiration from regularization theory \cite{2017-calatroni-etal,adler2017solving,2017-kobler-etal,adler2018learned,lunz2018adversarial,hammernik2018learning,hauptmann2020multi,2020-kobler-etal,li2020nett,mukherjee2020learned,bubba2021deep,gilton2021deep,mukherjee2021adversarially}: the physical model given by the forward map $F$ is assumed to be known while the regularizer (or the gradient updates related to it) is learned from a training set. While these approaches have shown impressive results in applications, a solid theory behind their successes is lacking. In view of the many sensitive applications where these methods are already being employed, e.g.\ in medical imaging, it is of utmost importance to fill this theoretical gap to better understand the strengths and limits of data-driven imaging modalities. Moreover, many inverse problems are naturally formulated in infinite-dimensional spaces \cite{engl1996,natterer2001mathematics,kaipio2006,epstein2007introduction,stuart-2010,monard2021consistent},  and their discretization must be carefully treated due to their ill-posedness \cite{kekkonen2014analysis}. Hence, it is of main interest to provide a theoretical analysis in the infinite-dimensional setting. \smallskip

In this work, we consider the problem of learning a regularizer for a linear inverse problem  in the framework of statistical learning theory, which is the natural setting to derive precise theoretical guarantees. This is part of the growing research area of learning an operator between infinite dimensional spaces \cite{tabaghi2019learning,de2019deep,lanthaler2021error,nelsen2021random,kovachki2021neural}. We study the case  where the measurements are modeled by a linear, possibly ill-posed, forward map, and the penalty term is a generalized Tikhonov regularizer \cite{tikhonov1963solution}.

More precisely, let $A\colon X\to Y$ be a bounded linear operator between the separable real Hilbert spaces $X$ and $Y$. We consider the inverse problem
\begin{equation}
y = A x + \epsilon ,
\label{eq:invprob} 
\end{equation}
which consists of the reconstruction of $x$ from the knowledge of $y$, where $\epsilon$ represents noise. We assume that $x$ is a random variable on $X$ with mean $\mu$ and covariance $\Sx$, and the noise $\varepsilon$, independent of the variable $x$, is a zero-mean random process on $Y$ with covariance $\Se$, (see Section~\ref{sec:stage} for more details). The operator $A$ is typically injective but its inverse may be unbounded: typical examples include denoising ($A$ is the identity) and deblurring ($A$ is a convolution operator). 

We aim to recover the unknown via generalized Tikhonov regularization. For a quadratic fidelity term $d_Y \colon Y \times Y \to \R$, the minimization problem
\begin{equation}\label{eq:almostthere0}
\min_x d_Y(Ax,y) + \| B^{-1} (x-\vh)\|_X^2,
\end{equation}
has a unique solution $R_{h,B}(y)$, called the generalized Tikhonov reconstruction. Here the pair $(h,B)$, where $h \in X$ and $B \colon X \to X$ is a positive bounded operator, is considered as a free parameter that we call \textit{regularization pair}. For example, the operator $B$ can be a smoothing operator in $L^2$, as a negative power of the Laplacian, so that $B^{-1}$ is a differential operator, yielding a classical regularization in Sobolev spaces.  We want to characterize and learn the optimal pair $(h,B)$ with respect to  the expected, or mean squared, error 
\[
L(h,B)=\mathbb{E}_{x,y}\|R_{h,B}(y)-x\|_X^2.
\]

Our first contribution is a complete characterization of the minimizers of $L$.
In particular we find that $(\mu, \Sx^{1/2})$ is a global minimizer (and is unique if $A$ is injective), which shows that the best regularizer is completely independent of the forward operator $A$ and depends only on the mean and covariance of $x$. This is consistent with the known linearized minimum mean squared error estimator in the finite-dimensional case \cite{kay1993fundamentals}, but it is usually not taken into account in the machine learning  approaches to inverse problems mentioned above. The extension to the infinite-dimensional case is not straightforward due to the presence of unbounded operators in inverse problems.

Since the computation of the expected error requires the full distribution $\rho$ of $x$ and $y$, we study how this can be approximated from a finite training set. We suppose to have access to a sample of $m$ pairs $\vzv= \{(x_j,y_j)\}_{j=1}^m$ drawn independently from $\rho$. In view of the results on the expected error, we consider two alternative ways to learn a regularizer pair $(\wh h_{\vzv},\wh B_{\vzv})$ from a training set $\vzv$: either by minimizing the empirical risk \cite{cucker2002mathematical} (supervised learning), or by using the empirical mean and covariance of $\{x_j\}$ (unsupervised learning).
In both cases, we prove generalization bounds for the sample error $|L(\wh\vh_{\vzv},\wh B_{\vzv}) - L(\mu,\Sx^{1/2})|$. Under some natural compactness assumptions on the class of regularization pairs, we prove that the sample error has the asymptotic behavior
\begin{equation} \label{eq:genasymp}
|L(\wh\vh_{\vzv},\wh B_{\vzv}) - L(\mu,\Sx^{1/2})| \lesssim \frac{1}{\sqrt m},
\end{equation}
with high probability, in both the supervised and the unsupervised approaches. We stress the point that these bounds hold in the infinite-dimensional setting, or, in other words, they do not depend on the discretization of the signal and of the measurements. 

Finally, we complement our theoretical findings with some numerical experiments. For a 1D denoising problem (i.e.\ $A$ is the identity operator) we replicate the asymptotic bound \eqref{eq:genasymp} at different discretization scales. Moreover, we find that the unsupervised approach, despite yielding the same rate \eqref{eq:genasymp}, clearly outperforms the supervised one.

The paper is organized as follows. In Section~\ref{sec:stage} we introduce the main notation and technical assumptions that will be used throughout the paper, including several examples. Section~\ref{sec:optimal} presents the main results for the minimization of the expected error, while Section~\ref{sec:sample} is devoted to the study of the sample error. Numerical experiments are the subject of Section~\ref{sec:numerics}. Concluding remarks and discussions are reserved for Section~\ref{sec:conclusions}.

\section{Setting the stage}
\label{sec:stage}


\subsection{The random objects $x$ and $\epsilon$}

As mentioned in the introduction, we formulate \eqref{eq:invprob} as a statistical inverse problem, where $x$ and $\epsilon$ are not deterministic but random. Let us start with the description of the prior on $x$.

\begin{assumption}
Let $x$ be a random variable on a probability space $(\Omega,\Prob)$ taking values in $X$. More precisely, 
 $x$ is square-integrable, so that its expectation $\mu \in X$  and its covariance $\Sx\colon X\to X$  is a trace-class operator. We assume that $\Sx$ is injective.
\label{ass:x_rand}
\end{assumption}
 Without loss of generality we can always assume that $\Sx$ is injective (i.e., $x$ is non-degenerate), since otherwise it would be enough to consider the inverse problem only in $(\ker \Sx)^\perp\subsetneq X$.

Let us consider some common examples of priors arising in inverse problems.
\begin{example}[Gaussian random variables]\label{ex:gauss}
A general class of priors arises when considering Gaussian random variables. We recall that $x$ is a Gaussian random variable if for all $v\in X$, $\langle x,v\rangle$ is a real Gaussian random variable and, by Fernique's theorem,   $x$ is square-integrable \cite{bogachev1998Gaussian}. Since  $\Sx\colon X\to X$ is self-adjoint, positive and trace-class, we can write its singular value decomposition (SVD) as
\[
\Sx v = \sum_{k} \sigma_k^2 \langle v,e_k\rangle_X e_k,\qquad v\in X,
\]
where $\{e_k\}_k$ is an orthonormal basis of $X$, $\sum_k \sigma_k^2 <+\infty$ and $\langle x,e_k\rangle\sim \mathcal{N}(\mu_k,\sigma_k)$, where $\mu=\sum_k\mu_k e_k$ is the mean of $x$.
 In other words,
 \begin{equation}\label{eq:x}
x=\mu+\sum_k \sigma_k a_k e_k,
 \end{equation}
where $a_k$ are i.i.d.\ standard Gaussian variables. This shows that, in infinite dimension, since $\sigma_k\to 0$, the variations of $x$ along the direction $e_k$ become smaller and smaller as $k\to+\infty$.
\end{example}
This abstract construction reduces to a smoothness prior by suitably choosing the covariance operator.

\begin{example}[Smoothing priors]\label{ex:laplacian}
Let $X=L^2(\mathbb{T}^d)$, where $\mathbb{T}^d=\R^d/\Z^d$ is the $d$-dimensional torus and $d\ge 1$. Let $\Delta$ denote the Laplace-Beltrami operator on $\mathbb{T}^d$, which is simply the classical Laplace operator on $[0,1]^d$ with periodic boundary conditions. For $s>\frac d2$, the operator
\[
(I-\Delta)^{-s}\colon L^2(\mathbb{T}^d)\to L^2(\mathbb{T}^d)
\]
is trace class, and can be used to define the Gaussian distribution $\mathcal{N}(0,(I-\Delta)^{-s})$. In the notation of Example~\ref{ex:gauss}, the SVD of $(I-\Delta)^{-s}$ is given by $\sigma_k^2=(1+4\pi|k|^2)^{-s}$ and $e_k(t)=e^{2\pi i k\cdot t}$ with $k\in\Z^d$. This enforces a smoothness prior on $x$, depending on the parameter $s$, which controls the decay of the Fourier coefficients of $x$ (see \cite[Appendix B]{saksman2009discretization} and \cite{bogachev1998Gaussian}).
\end{example}

Let us now discuss the model for the noise $\epsilon$.
\begin{assumption}\label{ass:e_rand}
Let $\varepsilon=(\eps)_{v\in Y}$ be a  (linear) random process on $Y$ with zero mean and such that its covariance  $\Se \colon Y\to Y$ defined by $\mathbb{E}[\varepsilon_v \varepsilon_w] = \langle \Se v,w\rangle_Y$
is bounded and injective.
\end{assumption}
Notice that the injectivity of $\Se$ implies that noise is present in all directions of $Y$, whereas the boundedness of $\Sigma_\eps$  allows us to regard the random process $\eps$ as a bounded (linear) operator from $Y$ into  $L^2(\Omega,\Prob)$.
The reader is referred to \cite{franklin1970well} and to Appendix \ref{app:processes} for additional details on random processes in Hilbert spaces. We mention here some basic properties that will be needed for the following discussion.
\begin{rem}\label{rem:Kiota}
Even if $\epsilon$ may not belong to $Y$ almost surely (see Example~\ref{ex:white} below), it is always possible to view it as an element of a larger space, as we now discuss.
Let $K$ be a separable Hilbert space and $\iota\colon K\to Y$ be an injective linear map such that $\iota(K)$ is dense in $Y$ and
\begin{equation}\label{eq:iotaSe}
\iota^*\circ \Se \circ\iota \colon K\to K^*\;\text{is trace-class,}\footnote{It is worth observing that $K$ and $\iota$ always exist: it is enough to choose them, independently of $\epsilon$, so that the embedding $\iota\colon K \rightarrow Y$ is  Hilbert-Schmidt, which implies that $\iota^*\circ \Se \circ\iota$ is trace class, since $\Se$ is bounded.}
\end{equation}
where 
we identify $Y^*=Y$, but we do not identify $K^*$ with $K$, and we regard $\iota^*$ as the canonical embedding  $Y\to K^*$.   The restriction of $\varepsilon$ to $K$ is a Hilbert-Schmidt operator from $K$ into $L^2(\Omega,\Prob)$, hence there exists a unique square-integrable random vector $\eps$ taking values in $K^*$ such that
$
\eps_v= \lK \eps, v \rK
$
for $v\in K$.
It is easy to show that the random vector $\eps$ has zero mean and its covariance
operator is  $\iota^*\circ \Se \circ\iota: K \rightarrow K^*$, since
\begin{equation}\label{eq:newcov}
\mathbb{E}[\lK \varepsilon, v \rK \lK \varepsilon, w \rK] =
\mathbb{E}[ \varepsilon_{\iota( v)}  \varepsilon_{\iota(w)} ]=
\langle \Se \iota v, \iota w \rangle_Y, \qquad v,w \in K.
\end{equation}
\end{rem}
A random variable is always a random process, as we now describe.
\begin{example}
A simple example of this abstract construction consists of considering a random variable $\epsilon$. In this case $\Se$ is trace-class itself, so that we can choose $K=Y$ and $\iota=I$ and $\epsilon\in Y$ almost surely. As discussed in Example~\ref{ex:gauss} for Gaussian variables, this means that, since $\sigma_k\to0$, the expected amplitude  of the noise in the direction $e_k$ goes to $0$ as $k\to \infty$. For instance, the choice of $\Se=(I-\Delta)^{-s}$  as in Example~\ref{ex:laplacian} corresponds to smaller noise levels for higher Fourier modes.
\end{example}
A random process allows for considering noise that is uniformly distributed in all directions.

\begin{example}[White noise]\label{ex:white}
The Gaussian white noise $\varepsilon$ is a random process  on $Y$ such that for any $v \in Y$ it holds that $\varepsilon_v$ is a standard Gaussian variable (mean $0$ and variance $1$), so that  $\Se= I$. Heuristically, in the notation of Example~\ref{ex:gauss}, this corresponds to $\sigma_k=1$ for every $k$, and so by \eqref{eq:x}
\[
\epsilon=\sum_k a_k e_k,\qquad a_k\sim \mathcal{N}(0,1),
\]
so that $\epsilon\notin Y$ with probability $1$ whenever $Y$ is infinite dimensional (see, e.g., \cite{franklin1970well}). In view of Remark~\ref{rem:Kiota}, it is possible to consider a larger space $K^*$ so that $\epsilon\in K^*$ almost surely. For concreteness of explanation, we focus on the case when $Y=L^2(\mathbb{T}^d)$, a typical framework in imaging. A possible choice for the space $K$ is the Sobolev space $H^s(\mathbb{T}^d)$ with $s>d/2$ (see \cite{kekkonen2014analysis}), so that the canonical embedding $\iota\colon H^s(\mathbb{T}^d)\to L^2(\mathbb{T}^d)$ is a Hilbert-Schmidt operator, hence \eqref{eq:iotaSe} is satisfied and $\epsilon$ can naturally be seen as an element of $H^{-s}(\mathbb{T}^d)=H^s(\mathbb{T}^d)^*$.
\end{example}

\subsection{The new formulation of the inverse problem and of the regularization}

As a consequence of Assumptions \ref{ass:e_rand}, since $\epsilon$ may not belong to $Y$, the inverse problem \eqref{eq:invprob} must be interpreted from a different perspective, namely considering $y$ as the stochastic process $
y_v = \lY Ax, v\rY + \varepsilon_v 
$ on $Y$
or by formulating the problem as an equation in $K^*$:
\[
y =  \iota^* Ax + \varepsilon,
\]
i.e.\ $\lK y, v \rK = \lY Ax,\iota(v) \rY + \lK \varepsilon, v \rK$ for $v \in K$, where $\iota^*\colon Y\to K^*$ is the natural embedding. We denote  the joint probability distribution of $(x,y)$ on $X\times K^*$ by $\rho$.

We now provide a consistent formulation of the quadratic functional appearing in \eqref{eq:almostthere0}. The goal is to replicate what would be the natural choice in a finite dimensional context, i.e.
\begin{equation}\label{eq:almostthere}
 \min_x \|\Se^{-1/2} (Ax-y) \|_Y^2 + \| B^{-1} (x-\vh)\|_X^2.
\end{equation}
Unfortunately, if $Y$ is infinite dimensional, the first factor is in general not well-defined, since for example for Gaussian processes $\epsilon\in\Im\Se^{1/2}$ with probability $0$ \cite{bogachev1998Gaussian}. Thus, we need to write this minimization problem in a different formulation.
We start by stating the assumptions we make on $B$.
\begin{assumption}\label{ass:SigmaAB}
Let us assume that  $B\colon X \rightarrow X$ is a bounded positive operator such that
\begin{equation}\label{eq:ass23}
  \operatorname{Im} (AB)\subseteq \operatorname{Im}(\Se\iota).
\end{equation}
\end{assumption}
 It is worth observing that, whenever $Y$ is infinite dimensional, then $\operatorname{Im}(\Se\iota)\subsetneq Y$ since $\Se\iota$ is compact. Furthermore, \eqref{eq:ass23} requires, in some sense, the operator $AB\colon X\to Y$ to be at least as ``smoothing'' as the operator $\Se\iota\colon K\to Y$. For instance, in the case when $AB$ and $\Sigma_\varepsilon\iota$ have the same left-singular vectors, this condition means that the singular values of $AB$  should go to $0$ at least as fast as the singular values of $\Sigma_\varepsilon\iota$. 
 


We are now ready to rewrite the functional  in \eqref{eq:almostthere}. The penalty term involving $B^{-1}(x-h)$ suggests the change of variables  $x=h+Bx'$. The corresponding minimization problem for $x'$ reads
\begin{equation}\label{eq:almostthere1}
 \min_{x'\in X} \|\Se^{-1/2} (A(h+Bx')-y) \|_Y^2 + \| x'\|_X^2.
\end{equation}
This expression does not require the injectivity of $B$.
By formally expanding the first factor we obtain
\[
    \|\Se^{-1/2} ABx' \|_Y^2 -2\langle  \Se^{-1/2} (y-Ah),\Se^{-1/2} ABx'\rangle_Y+ \|\Se^{-1/2} (y-Ah) \|_Y^2 .
\]
Let us analyze these three terms separately:
\begin{enumerate}[\hspace{0mm} 1.]
    \item Since $\Se$ is self-adjoint we have $\|\Se^{-1/2} ABx' \|_Y^2=\langle \Se^{-1} ABx',AB'x\rangle_Y$, which is well-defined  because $\operatorname{Im} (AB)\subseteq \operatorname{Im}(\Se)$ thanks to \eqref{eq:ass23}.
    \item The second factor is formally equivalent to $-2\langle   y-Ah,\Se^{-1} ABx'\rangle_Y$, which is not well-defined as scalar product in $Y$ since $y$ may not belong to $Y$. However, since $y\in K^*$ and $\Se^{-1} ABx'\in\iota(K)$ by \eqref{eq:ass23}, this scalar product can be interpreted as the duality pairing
    $
    -2\lK y-\iota^*Ah, \iota^{-1}\Se^{-1}ABx' \rK
    $.
    \item The third factor $\|\Se^{-1/2} (y-Ah) \|_Y^2 $ is independent of $x$, and so it irrelevant for the minimization task: thus, we remove it.  This is a key step in  infinite dimension, since, as mentioned above,  $\|\Se^{-1/2} y \|_Y^2=+\infty$ almost surely. See  \cite[Remark~3.8]{stuart-2010} for additional details on this aspect.
\end{enumerate}
This discussion motivates the introduction of the following functional, formally equivalent to \eqref{eq:almostthere1}.
For  $y \in K^*$, we define the regularized
 solution of the inverse problem as $\hat x=h+B\hat x'$, where
\begin{equation}
\hat x' = \argmin_{x' \in X} \| \Se^{-1/2} ABx' \|_Y^2 - 2 \lK y-\iota^*Ah, (\Se\iota)^{-1}ABx' \rK + \| x'\|_X^2.
\label{eq:hatx}
\end{equation}
The minimizer exists and is unique, and gives the following expression for $\hat x$:
\begin{equation}
R_{h,B}(y):=h+B\hat x' =h+ B (BA^* \Se^{-1} A B + I)^{-1}((\Se\iota)^{-1}AB)^* (y-\iota^*Ah),
    \label{eq:affine}
\end{equation}
where $R_{h,B}\colon K^*\to X$ is a bounded affine map.
See Proposition~\ref{prop:quad_min} for all the details. Note that $((\Se\iota)^{-1}AB)^*\colon K^*\to X$ is well-defined thanks to \eqref{eq:ass23}.

\section{The optimal regularizer}
\label{sec:optimal}

The regularization approach described so far is based on the choice of $h$ and $B$. In classical regularization theory, these are chosen depending on the prior knowledge of the problem under consideration. In the data-driven approach we consider in this work, $h$ and $B$ are learned from training data.
In this section, we let learning come into play and consider the problem of determining the optimal $h$ and $B$, under the assumptions that the distributions of $x$ and $\epsilon$ are fully known. More precisely, this allows for the explicit computation of the expected error
\[
L(h,B)=\mathbb{E}_{(x,y)\sim\rho}\|R_{h,B}(y)-x\|_X^2
=\mathbb{E}_{x,\epsilon}\|R_{h,B}(\iota^*Ax+\epsilon)-x\|_X^2,
\]
which quantifies the mean square error that our regularization functional \eqref{eq:hatx} yields. Optimal choices of $\hOp$ and $\BOp$ are those that minimize $L(h,B)$, and are characterized in the following result.

\begin{thm}\label{prop:quadratic_target_inf2}
Let $X$ and $Y$ and be separable real Hilbert spaces, $A\colon X\to Y$ be a bounded linear operator, $x$ and $\epsilon$ satisfy Assumptions~\ref{ass:x_rand} and \ref{ass:e_rand} and be independent, and $K$ and $\iota$ be as in Remark~\ref{rem:Kiota}.
Suppose  $B = \Sx^{1/2}$ satisfies Assumption~\ref{ass:SigmaAB}.

Consider the minimization problem
\begin{equation}\label{eq:min2}
\min_{h,B} \{ \E_{(x,y)\sim \rho}\left[ \| R_{h,B} (y) - x \|_{X}^2\right] \} ,
\end{equation}
where the minimum is taken over all $B$ satisfying Assumption~\ref{ass:SigmaAB} and over all $h\in X$. Then $(\BOp,\hOp)$ is a global minimizer of \eqref{eq:min2} if and only if
\[
\hOp=\mu\qquad\text{and}\qquad 
B^2|_{(\ker A)^\perp} = \Sx|_{(\ker A)^\perp}.
\]
In particular,  $\BOp=\Sx^{1/2}$ is always a global minimizer, and is unique if $A$ is injective. Furthermore, for every minimizer $(\hOp,\BOp)$, the corresponding reconstruction map is independent of $\BOp$ and, for all  $y\in K^*$, is given by
\begin{align}
R^\star(y) &=    \mu + \Sx^{1/2}(\Sx^{1/2}A^* \Se^{-1} A \Sx^{1/2} + I_X)^{-1}((\Se\iota)^{-1}A\Sx^{1/2})^* (y -\iota^*A \mu)\label{eq:R1}\\
&=\mu + \Sx A^*(\iota^* (A\Sx A^* + \Se)  )^{-1}(y-\iota^*A\mu). \label{eq:R2}
\end{align}
\end{thm}

The proof is in Appendix~\ref{sec:app3}. Some comments on this result are in order.
\vspace{-1mm}
\begin{itemize}
\setlength\itemsep{0em}
\item By assumption $\iota^* (A\Sx A^* + \Se)$ is an injective compact operator from $Y$ to $K^*$, so that its inverse is not bounded, however it is possible to prove that $\Sx A^*(\iota^* (A\Sx A^* + \Se)  )^{-1}$ extends to a bounded operator from $K^*$ into $X$. With a slight abuse of notation, we denote this extension in the same way, so that~\eqref{eq:R2} makes sense for all $y\in K^*$.
\item To prove this result, we first consider the minimization in \eqref{eq:min2} over all possibile affine maps, which yields the so-called Linearized Minimum Mean Square Error (LMMSE) estimator of $x$.
Then, it is possible to show that such optimal affine functional is of the form $R_{\hOp,\BOp}$, for suitable $\BOp$ and $\hOp$. 
In a finite-dimensional context, such a result is a direct consequence of the expression of the LMMSE estimator (see, e.g., \cite[Theorem 12.1]{kay1993fundamentals}). Theorem~\ref{prop:quadratic_target_inf2} generalizes this result to the infinite-dimensional case.
\item In the case of Gaussian random variables, the expression of the optimal regularizer $R^\star$ coincides with the maximum a posteriori (MAP) estimator.
Nevertheless, our result does not require any assumptions on $x$ and $\varepsilon$ being Gaussian (see the discussion in \cite{gribonval2011should,gribonval2013reconciling}).

    \item The minimum expected loss can be computed as
     \begin{equation}
L(\hOp,\BOp) = \tr\big( \Sx^{1/2}(\Sx^{1/2}A^*\Se^{-1}A\Sx^{1/2} + I_X)^{-1} \Sx^{1/2} \big),
\label{eq:minimum}        
    \end{equation}
as it is reported in Appendix~\ref{sec:app3}.
\item It is worth observing that the optimal regularization parameters $\BOp=\Sx^{1/2}$ and $\hOp=\mu$ are independent of $A$ and $\epsilon$, and depend only on the mean and the covariance of $x$.
\end{itemize}

\section{Finding the optimal regularizer: the sample error}
\label{sec:sample}

The computation of the optimal regularizer proposed in the previous section through the minimization of the expected loss $L$ requires the knowledge of the joint probability distribution $\rho$ of $x$ and $y$. In this section, we suppose that $\rho$ is unknown, \footnote{More precisely, we only assume that $\Sigma_\eps$ is known.}
but we have access to a sample $\vzv = \{(x_j,y_j)\}_{j=1}^{m}$ of $m$ pairs $(x_j,y_j) \in Z = X \times K^*$ drawn independently from the joint probability distribution $\rho$, and we study how to learn an estimator $(\wh h_{\vzv},\wh B_{\vzv})$ of the optimal parameters $(\hOp,\BOp)$. 
We propose two alternative ways to learn an estimator based on a training sample $\vzv$.
For the ease of notation, from now on we omit the dependence  on $\vzv$.
\begin{enumerate}[1.]
    \item \textit{Supervised learning}:  $(\hS,\BS)$ is determined by minimizing the empirical risk $\widehat{L}$, namely 
    \begin{equation}
    (\hS,\BS)=\operatornamewithlimits{arg min}_{(h,B)\in \Theta}\, \widehat{L}(\vh,B), \qquad \widehat{L}(\vh,B) = \frac{1}{m} \sum_{j=1}^m \| R_{\vh,B} (y_j) - x_j \|_{X}^2 ,
        \label{eq:empirical_risk}
    \end{equation}
    where $\Theta$ is a suitable subset of $X \times \mathcal{L}(X,X)$. 
    \item \textit{Unsupervised learning}: since the best parameters are $\hOp = \mu$ and $\BOp=\Sx^{1/2}$, a natural estimator is 
    provided by means of the sample $\{x_j\}$ alone as follows:
    \begin{equation}
    \hU = \widehat{\mu} = \frac{1}{m} \sum_{j = 1}^m x_j, \qquad \BU = \widehat{\Sx}^{1/2}, \quad \widehat{\Sx} = \frac{1}{m} \sum_{j=1}^m (x_j - \hat{\mu})\otimes (x_j - \hat{\mu}).
    \label{eq:samp_avg_cov}    
    \end{equation}
\end{enumerate}
In both cases, we evaluate the  quality of  $(\wh h,\wh B)$  in terms of its \textit{excess error}  $L(\wh h,\wh B) - L(\hOp,\BOp)$. 

\subsection{Supervised learning: empirical risk minimization}\label{sec:supervised}

There exist several techniques to show the convergence of the empirical risk minimizer to the optimal parameter, involving tools such as the VC dimension and the Rademacher complexity (see, e.g., \cite{shalev2014}), which require some compactness assumption on $\Theta$. Here, we fix a Hilbert space $H$ with a compact embedding 
$j\colon H \rightarrow X$ having dense range. For $\varrho_1>0$, set
\begin{equation}
\Theta_1 = \{j(\spec{h}) \colon \spec{h} \in H, \| \spec{h} \|_H \leq \varrho_1\}, \;  \Theta_2 = \{j\spec{B}j^*\colon \spec{B} \in \operatorname{HS}(H^*,H), \| \spec{B} \|_{\operatorname{HS}(H^*,H)} \leq \varrho_1\},
\label{eq:Theta}
\end{equation} 
and define $\Theta$  as the set of pairs $\{(h,B)\in \Theta_1 \times \Theta_2: \text{$B$ is positive}\}$. Here,  $\operatorname{HS}(H^*,H)$ denotes the space of Hilbert-Schmidt operators from $H^*$ to $H$. We further assume that
\begin{enumerate}[a)]
    \item \label{ass:a} the map $j$ can be decomposed as $j = j_2 \circ j_1$, where $j_1 \colon H \to X$ and $j_2 \colon X \to X$ are compact and satisfy
		\begin{equation}\label{eq:poly}
			s_k(j_1)\lesssim  k^{-s}, \qquad s>0, \qquad \text{being $s_k(j_1)$ the singular values of $j_1$};
		\end{equation}
		whereas $j_2$ is such that
		\begin{equation}\label{eq:2.8new}
			\operatorname{Im}(Aj_2) \subseteq \operatorname{Im}(\Se \iota).
		\end{equation}
\item \label{ass:b} The optimal parameter $(\hOp,\BOp)=(\mu,\Sx^{1/2})$ belongs to  $\Theta$.
\end{enumerate}
Assumption~\ref{ass:a}), and in particular \eqref{eq:poly}, allows us to explicitly compute the covering numbers, whereas \eqref{eq:2.8new} ensures that Assumption~\ref{ass:SigmaAB} holds uniformly for each positive operator $B\in\Theta_2$.
For example, when $H = H^{\sigma_1}(\T^1)$ and $X = H^{\sigma_2}(\T^1)$ are Sobolev spaces on the one-dimensional torus, assumption \ref{ass:a}) is satisfied if $s = \sigma_1 - \sigma_2 > 0$. As a consequence, assumption~\ref{ass:b}) can be interpreted as an \textit{a priori} regularity assumption on the problem. Such hypothesis can be relaxed by introducing the approximation error, namely, the rate at which the space $\Theta$ approximates $X \times \mathcal{L}(X,X)$ as the radius $\varrho_1$ grows to $\infty$. Such an analysis, which easily follows from the range-density property of $j$, is not treated here.

Finally, we assume that both the inputs and the outputs are bounded, {\em i.e.}
\begin{equation}
\operatorname{supp}(\rho) \subset B_Z(\varrho_2), \text{ a ball of $Z=X \times K^*$ of radius $\varrho_2$}.
    \label{eq:bound_supp}
\end{equation}
\begin{thm}
\label{thm:gener_sup} Under the above conditions, let $(\hS,\BS)$ be defined by \eqref{eq:empirical_risk} and take $\tau>0$. We have 
\begin{equation}
\label{eq:generalization_sup}    
 |L(\hS,\BS) - L(\hOp,\BOp)| \leq \left( \frac{ c_1 + c_2 \sqrt{\tau}}{\sqrt{m}} \right)^{1 - \frac{1}{2s'+1}},\qquad m\geq m_0,
\end{equation}
with probability exceeding $1-e^{-\tau}$, being $0 < s'< s$, where $c_1,c_2,m_0$ are independent of $m$ and $\tau$.
\end{thm}
The proof of Theorem \ref{thm:gener_sup} is  reported in the Appendices~\ref{app:proof_supervised} and \ref{app:entropy}. The approach  is inspired by \cite[Proposition 4]{cucker2002mathematical} and is suited for a much broader class of learning problems: by adapting Lemma \ref{lem:quadratic_reg}, it is possible to extend the current approach to non-quadratic regularization functionals.
\par
A prominent example of $H$ satisfying \eqref{eq:poly} comes from Sobolev spaces. Consider, e.g., $X=L^2(\mathbb{T}^d)$, where $\mathbb{T}^d$ is the $d$-dimensional torus, and $H = H^\sigma(\mathbb{T}^d)$. If $\sigma >0$, the embedding of $H$ in $X$ is compact, and its singular values show a polynomial decay \eqref{eq:poly} with $s = \sigma/d$.

\subsection{Unsupervised learning: empirical mean and covariance}\label{sec:unsup}

As pointed out in \eqref{eq:samp_avg_cov}, it is possible to recover an approximation of the optimal parameter $(\hOp$, $\BOp)$ only by taking advantage of a sample of the output variable $\{x_j\}_{j=1}^m$. 
Since this technique does not require matched couples of inputs and outputs, we refer to it as an unsupervised learning approach.
In order to prove a bound in probability for the sample error, in this section we assume that $x$ is a $\kappa$-sub-Gaussian random vector, i.e.,
\begin{equation}
\forall v \in X, \langle x,v\rangle_X \text{ is a real sub-Gaussian r.v., i.e. } \| \langle x,v\rangle_X\|_p \leq \kappa \sqrt{p}\| \langle x,v\rangle_X\|_2, \ \forall p>1,
\label{eq:subGaussian}
\end{equation}
where $\|\langle x,v\rangle\|_p^p=\mathbb E[ |\langle x,v\rangle|^p]$. It its known \cite{vershynin2018high} that $\mathbb E[\|x\|^p]$ is finite for all $p>0$, so that $x$ has finite mean and its covariance operator $\Sigma_x$ is trace-class.  Gaussian random variables are a particular instance of sub-Gaussian random variables by Fernique's theorem \cite{bogachev1998Gaussian}. Note that, in infinite-dimensional  spaces, bounded random vectors in general are not sub-Gaussian. 

We further assume that 
the injective operator $\Sigma_\eps$ has a bounded inverse, thus $\Sigma_\eps + A\Sigma_xA^*$ is invertible. This is satisfied for example if $\eps$ is the white-noise,  since $\Sigma_\eps=\operatorname{I}$.
We also require that $A^* (\iota^*(\Sigma_\eps +
  A\Sigma_x A^*))^{-1}$, defined on $\iota^*(Y)\subset K^*$, extends to a bounded operator from $K^*$ into $X$.
\begin{thm}
\label{thm:gener_unsup}
Under the above conditions, let $(\hU,\BU)$ be defined by \eqref{eq:samp_avg_cov} and take  $\tau>0$. Then,
\begin{equation}
\label{eq:generalization_sup2}    
 |L(\hU,\BU) - L(\hOp,\BOp)| \leq \frac{c_3 +
   c_4\sqrt{\tau}}{\sqrt{m}},\qquad m\geq m_0,
\end{equation}
with probability exceeding  $1-e^{-\tau}$, where $m_0$,  $c_3$ and $c_4$ depend only on $\Sx$, $\Se$ and $A$.
\end{thm}
The proof of Theorem \ref{thm:gener_unsup} is based on several concentration estimates reported in Appendix~\ref{app:proof_unsupervised}. 
The rates we obtain can be meaningfully compared with recent results in supervised learning: see \cite{blanchard2018optimal, lin2020optimal}.

\section{Numerical simulations}
\label{sec:numerics}

We report some numerical results obtained from the supervised and unsupervised strategies for a denoising problem, using synthetic data.
The goal of these experiments is twofold: on one hand, we want to study the asymptotic properties of the regularizers learned with  the techniques proposed in Section \ref{sec:sample} as the sample size $m$ grows, verifying Theorems~\ref{thm:gener_sup} and \ref{thm:gener_unsup}. On the other hand, we want to assess that those properties, obtained in an infinite-dimensional setting, do not suffer from the curse of dimensionality. We do so by introducing finer and finer discretizations, and showing that the theoretical bounds do not degrade as the dimension of the problem increases.

\subsection{Problem formulation and discretization}
We consider a denoising problem on $X = Y = L^2(\mathbb{T}^1)$, being $\mathbb{T}^1 = \R/\Z$ the one-dimensional torus, which consists in determining a signal $x$ from the noisy measurement $y = x + \varepsilon$ and thus corresponds to the case $A = I$. We define a statistical model both for $\varepsilon$ and for $x$, which we use for the generation of the training data. In the learning process, though, we do not take advantage of the knowledge of the introduced probability distributions, apart from the covariance operator of the noise $\Se$.
In accordance with Assumption \ref{ass:e_rand}, we assume that $\varepsilon$ is a random process on $Y$, and in particular we consider a white noise process, i.e. with zero mean and $\Se = \sigma^2 I$. We consider a noise level of 5\%, namely, the standard deviation $\sigma$ is set to the $5\%$ of the peak value of the average signal. In different tests, we employ different white noise processes with different distributions, including the Gaussian (cfr.\ Example~\ref{ex:white}) and the uniform distributions. Regarding  $x$, we assume  a Gaussian distribution with fixed mean $\mu$ and covariance $\Sx$, where $\mu = 1 - |2x-1|$ and $\Sx^{1/2}$ is  a convolution operator.

In order to discretize the described problem, we fix $N>0$ and approximate the space $X$ by means of the $N-$dimensional space  generated by a 1D-pixel basis. As a consequence, the functions in $X$ and $Y$ are approximated by vectors in $\R^N$, and the linear operators by matrices in $\R^{N \times N}$. More details on the discretizations and on the random process generation are reported in Appendix \ref{app:numerics}.

\subsection{Implementation and results}
We denote $\theta=(h,B)$. The workflow of the numerical experiments is described as follows:
\vspace{-1mm}
\begin{enumerate}
	\item Fix the discretization size $N$, define the optimal regularizer $R_{\thetaOp}$. Compute $L(\thetaOp)$.
	\item For each selected value of the sample size $m$,
	\begin{itemize}\setlength\itemsep{0em}
		\item generate the samples $\{x_j\}_{j=1}^m$, $\{\varepsilon_j\}_{j=1}^m$;
		\item minimize the empirical risk $\wh{L}(\theta)$ to find $\thetaS$;
		\item compute the empirical mean and covariance $\wh{\mu}$, $\wh{\Sx}$ to find $\thetaU$;
		\item compute the excess risks $|L(\thetaS) - L(\thetaOp)|$ and $|L(\thetaU) - L(\thetaOp)|$.
	\end{itemize}
	\item Show the decay of both computed quantities as $m$ increases.
\end{enumerate}
\vspace{-1mm}

We compute the mean squared errors $L(\thetaOp)$, $L(\thetaS)$ and $L(\thetaU)$ according to the definition of $L$, thus avoiding the use of a test set. Moreover, we perform the minimization of the empirical risk analytically, thanks to the explicit expression of the regularization functional provided by \eqref{eq:affine} (see Appendix \ref{app:numerics}). As a final remark, the generalization bounds in Theorems \ref{thm:gener_sup} and \ref{thm:gener_unsup} can be reformulated in expectation.
Thus, to verify the expected decay, we repeat the same experiment $30$ times, with different training samples for each size $m$ and taking the average in each repetition.\footnote{All computations were implemented with Matlab R2019a, running on a laptop with 16GB of RAM and 2.2 GHz Intel Core i7 CPU. All the codes are available at \url{https://github.com/LearnTikhonov/Code}}
\par

In Figure \ref{fig:1}, we present the outcome of the numerical experiments, conducted under different statistical models for $x$ and $\varepsilon$. The sample size ranges between $3 \cdot 10^3$ and $3\cdot 10^5$. In all the presented scenarios, the decay of the excess risk both in the supervised and in the unsupervised cases agrees with the theoretical estimates, showing a decay of the order $1/\sqrt{m}$. 
\begin{figure}
\centering
\begin{tabular}{ccc}
\includegraphics[width=0.32\columnwidth,trim={1cm 0cm 1cm 0cm}]{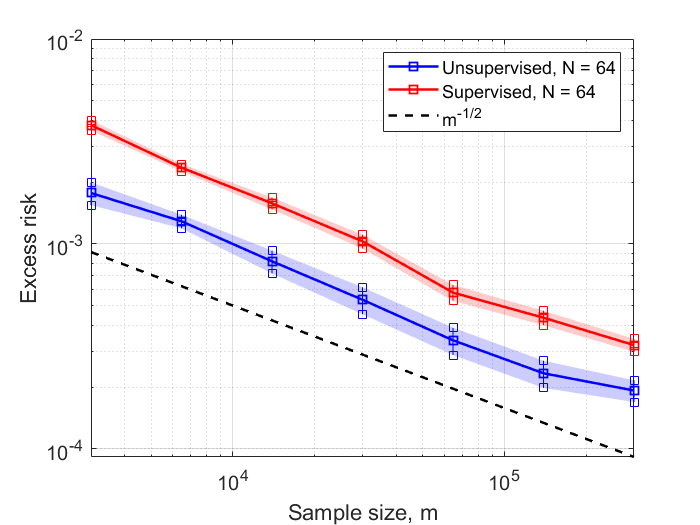}
& \includegraphics[width=0.32\columnwidth,trim={1cm 0cm 1cm 0cm}]{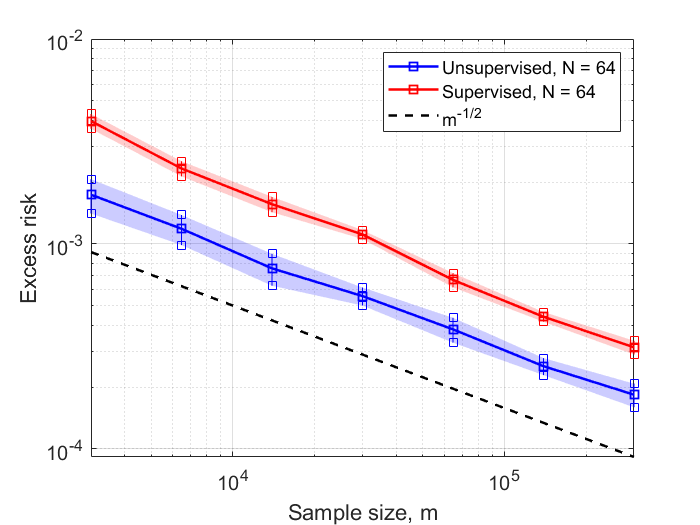} 
& \includegraphics[width=0.32\columnwidth,trim={1cm 0cm 1cm 0cm}]{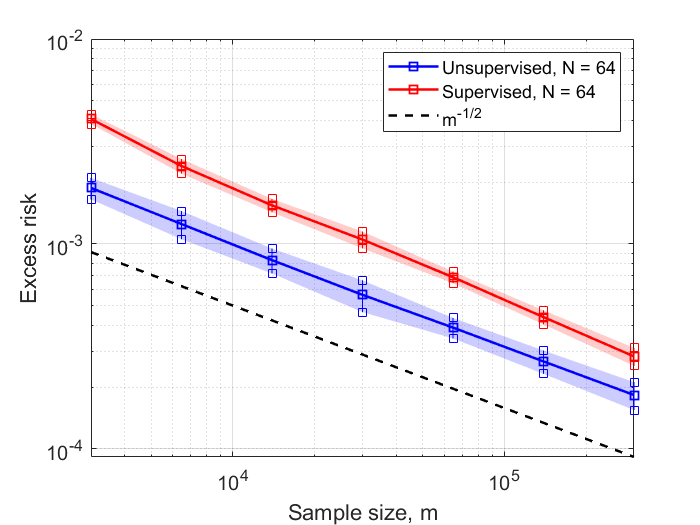} \\
(a) & (b) & (c)
\end{tabular}
\caption{Decay of the excess risks \textcolor{red}{$|L(\thetaS)-L(\thetaOp)|$} and \textcolor{blue}{$|L(\thetaU)-L(\thetaOp)|$} in three different cases: Gaussian variable $x$ and (a) Gaussian white noise $\varepsilon$, (b) uniform white noise $\varepsilon$, and (c) white noise $\varepsilon$ uniformly distributed w.r.t. the Haar wavelet transform. We also report standard deviation error bars.}%
\label{fig:1}%
\end{figure}
Finally, in Figure \ref{fig:2}, we show that the theoretical results are equivalently matched by numerics when the discretization size is increased. 
\begin{figure}
\centering
\begin{tabular}{ccc}
\includegraphics[width=0.32\columnwidth,trim={1cm 0cm 1cm 0cm}]{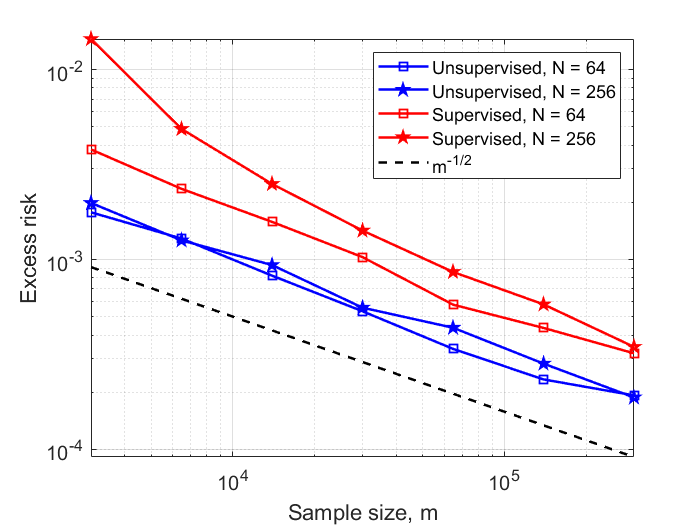}
& \includegraphics[width=0.32\columnwidth,trim={1cm 0cm 1cm 0cm}]{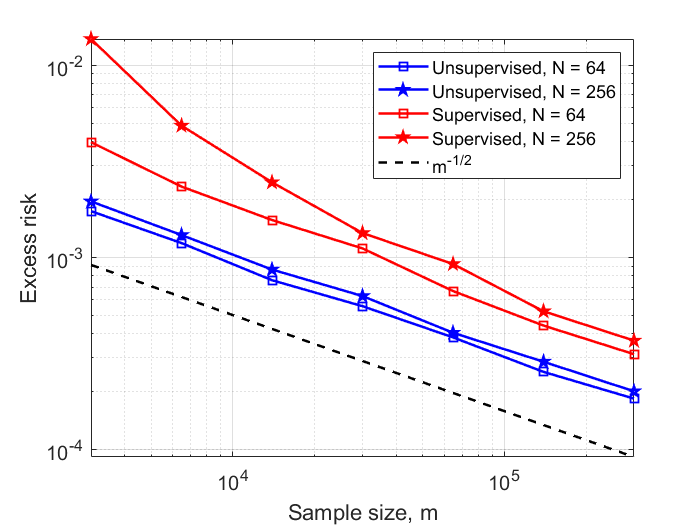} 
& \includegraphics[width=0.32\columnwidth,trim={1cm 0cm 1cm 0cm}]{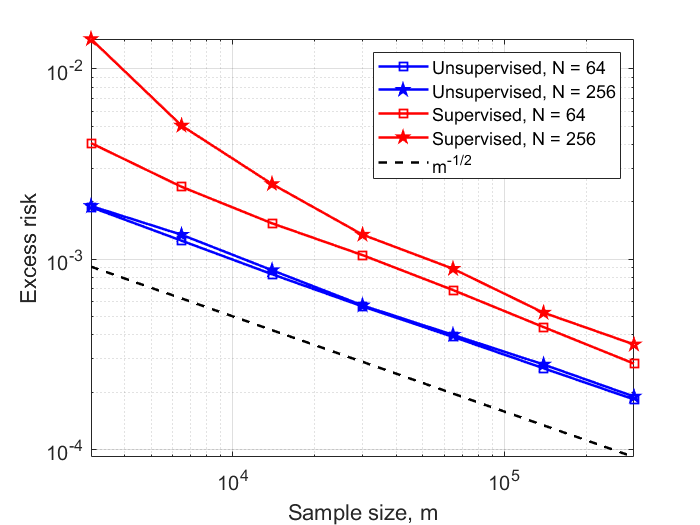} \\
(a) & (b) & (c)
\end{tabular}
\caption{Comparison of decay the excess risks \textcolor{red}{$|L(\thetaS)-L(\thetaOp)|$} and \textcolor{blue}{$|L(\thetaU)-L(\thetaOp)|$} with $N=64$ and $N=256$, in same three statistical models as in Figure \ref{fig:1}.}
\label{fig:2}%
\end{figure}

Additional details regarding the results of the numerical experiments are reported in Appendix \ref{app:numerics}. Moreover, in Appendix \ref{app:deconv} we replicate the presented numerical study for a different example, namely, a deconvolution problem for 1D signals. In this case, $A$ is a convolution operator with respect to a continuous kernel, whose inverse is in general unbounded.

\section{Conclusions and limitations}
\label{sec:conclusions}
We studied the problem of learning a regularization functional for an inverse problem between infinite dimensional spaces. This problem has received huge interest in recent years due to the successes in several imaging modalities. Our work provides theoretical support to machine learning approaches in sensitive applications such as medical imaging.  

We have considered the case of a linear inverse problem that is solved via generalized Tikhonov regularization. This involves an unknown operator $B$ and a signal $h$, both to be learned from data. We proposed two learning strategies, one supervised and one unsupervised. Surprisingly, we found that the regularizer learned with the unsupervised strategy has the same (or slightly better) generalization bounds than the supervised one. Furthermore, the unsupervised approach does not need the knowledge of the forward operator $A$ nor that of the distribution of the noise $\varepsilon$. This motivates the development of more advanced unsupervised approaches to the problem, e.g.\ with deep learning methods (see \cite{2017-kobler-etal,2020-kobler-etal,li2020nett,lunz2018adversarial,mukherjee2020learned}).

The analysis presented here was possible thanks to the simple form of the regularizer. Our work does not cover, for instance, the case of sparsity promoting regularization functionals \cite{grasmair2008sparse} or more general convex or non-convex penalty terms arising from deep learning methods. Some results regarding optimal
(non-quadratic) regularizers associated with different priors can be found e.g. in \cite{gribonval2011should,gribonval2013reconciling,burger2014maximum}, which nevertheless deal with a finite-dimensional setting.
Extensions to more general regularizers will be the subject of future studies.

\section*{Acknowledgments}
This material is based upon work supported by the Air Force Office of Scientific Research under award number FA8655-20-1-7027. GSA, EDV, LR and MS are members of the ``Gruppo Nazionale per l'Analisi Matematica, la Probabilit\`a e le loro Applicazioni'' (GNAMPA), of the ``Istituto Nazionale per l'Alta Matematica'' (INdAM). GSA is supported by a UniGe starting grant ``curiosity driven''. ML is supported by Academy of Finland, grants 273979 and 284715.

\bibliography{references}

\vfill
\pagebreak

\appendix

\section{Appendix}

\subsection{Random processes}
\label{app:processes}
The notion of random vector in infinite-dimensional vector spaces is not general enough to describe many models of noise, as for example the white noise described in Example~\ref{ex:white}. To overcome this problem, a possibility is to consider the noise as a \textit{random process} on $Y$ (see the approach  in \cite{franklin1970well}).
A random process is a collection  $\{\varepsilon_v\}_{v \in Y}$ of real random variables $\eps_v$, each of them defined on the same probability space $(\Omega,
\mathcal F,\Prob)$ and labelled by vectors $v\in Y$. Here we assume that the random process is linear, with zero mean and bounded. This means that 
\[
\varepsilon_{\alpha v+ \beta w} = \alpha \varepsilon_v + \beta \varepsilon_w, \qquad v,w \in Y, 
\]
and for each $v\in Y$, $\eps_v$ has zero-mean with finite bounded variance 
\begin{equation}\label{eq:boundeness}
 \mathbb E[ \eps_v^2]\leq C_\eps \|v\|^2 ,    
\end{equation}
where $C_\eps>0$ is a suitable constant independent of $v$. The existence of $C_\eps$ is equivalent to assuming that 
the covariance operator $\Se\colon Y \rightarrow Y$, given by
\[
\mathbb{E}[\varepsilon_v \varepsilon_w] = \langle \Se v,w \rangle_Y, \qquad v,w \in Y,
\]
is bounded from $Y$ to $Y$. It is easy to show that if $\eps\colon\Omega\to Y$ is a square-integrable random vector, then the collection of real random  variables $\{\varepsilon_v\}_{v \in Y}$
\begin{equation}\label{eq:weak-random}
    \eps_v(\omega)= \langle \eps(\omega),v\rangle_Y
\end{equation}
is a linear bounded random process. However, the converse is not true as shown by the following example. 
\begin{example}
The Gaussian white noise $\varepsilon$ on $Y$ is a random process  such that for any $v \in Y$ it holds that $\varepsilon(v)$ is a zero mean Gaussian variable, and  $\Se= I$, i.e., $\mathbb{E}[\varepsilon(v_i) \varepsilon(v_j)] = \langle v_i,v_j\rangle_Y$. Suppose now that $\varepsilon \in Y^*$: then, by Riesz representation theorem there should exist $\hat{\varepsilon} \in Y$ s.t. $\varepsilon(v) = \langle \hat{\varepsilon},v \rangle_Y$. Nevertheless, this leads to a contradiction, since, letting $\{\phi_i\}_i$ be an orthonormal basis of $Y$,
\[
\mathbb{E}[\| \varepsilon \|_{Y^*}^2] = \mathbb{E}[\| \hat{\varepsilon} \|_{Y}^2] = \sum_{i,j \in \N} \mathbb{E}[\langle \hat{\varepsilon}, \phi_i\rangle_Y \langle \hat{\varepsilon},\phi_j \rangle_Y] = \sum_{i,j \in \N} \mathbb{E}[\varepsilon(\phi_i) \varepsilon(\phi_j)] = \sum_{i,j \in \N} \langle \phi_i, \phi_j \rangle_Y,
\]
 which is a divergent sum. It is moreover easy to show that $\mathbb{P}(\| \varepsilon\|_{Y^*} < \infty) = 0$ (see, e.g. \cite{franklin1970well}).
\end{example}
However, given a random process $\{\varepsilon_v\}_{v \in Y}$, it is always possible to define a Gelfand triple $K\subseteq Y\subseteq K^*$ and random vector $\eps$ taking value in $K^*$ such that~\eqref{eq:weak-random} holds true for all $v\in K$. 

Indeed, let $K$ be a Hilbert space with a continuous embedding $\iota\colon K \to Y$ such that
$\iota(K)$ is dense in $Y$ and the linear map 
\[
  K\ni v \mapsto \eps_{\iota(v)} \in L^2(\Omega,\Prob)
\]
is a Hilbert-Schmidt operator. Observe that~\eqref{eq:boundeness} implies that the linear map
\[
Y\ni v \mapsto \eps_{v} \in L^2(\Omega,\Prob)
\]
is always bounded, so that it is enough to assume that $\iota$ is itself a Hilbert-Schmidt operator.  
To construct the Gelfand triple, we identify $Y^*$ with $Y$, but we do not identify $K^*$ with $K$. Hence, since $\iota(Y)$ is dense in $Y$, then $\iota^*\colon Y\to K^*$ is injective and $Y$ can be regarded as a (dense) subspace of $K^*$, so that $K\subseteq Y\subseteq K^*$. 

For a fixed $\iota$, the canonical identification 
\[ \operatorname{HS}(K,L^2(\Omega,\Prob))\simeq L^2(\Omega,\Prob) \otimes K^*\simeq L^2(\Omega,\Prob,K^*)
\]
implies that there exists $\eps\in L^2(\Omega,\Prob,K^*)$, {\em i.e.} a square-integrable random vector $\eps$ in $K^*$, such that
\[
\eps_v(\omega)= \langle \eps(\omega),v\rangle_{K^*\times K}
\]
almost surely.  
It is easy to show that the random vector $\eps\in K^*$ has zero mean and its covariance operator $\wt{\Sigma}_\eps$, defined as 
\[ \wt{\Sigma}_\eps\colon K\to K^* \qquad 
\langle \wt{\Sigma}_\eps v,  w \rangle_{K^*\times K} = \mathbb{E}[\lK \varepsilon, v \rK \lK \varepsilon, w \rK]  \qquad v,w \in K,
\] 
is given by $\wt{\Sigma}_\eps=\iota^*\circ \Se \circ\iota$, as shown by~\eqref{eq:newcov}.

For example, for the white noise on the Hilbert space $Y=L^2(D)$, $D \subset \R^d$, a possible choice is $K = H^s(D)$ with $s>d/2$, so that the embedding is a Hilbert-Schmidt operator (see \cite{kekkonen2014analysis}). Furthermore, it is possible to show that $\eps$ is a Gaussian random vector in $K^*$ and the space $Y\subseteq K^*$ is the corresponding Cameron-Martin space, so that $\Prob[\eps\in Y]=0$ (see \cite{bogachev1998Gaussian}).

Finally, observe that if the random process is already a  square-integrable random vector of $Y$,  then $\Sigma_\eps$ is a trace-class operator and we can simply choose $K = Y$. Hence the random process setting extends the usual formalism of random variables. Clearly, if $Y$ is  finite dimensional, the two approaches are equivalent.

\subsection{The solution of the regularization problem with fixed $h$ and $B$}
Throughout this section and the following ones, we denote the adjoint of an operator between Hilbert spaces $F \colon \mathcal{H}_1 \to \mathcal{H}_2$ as $F^*\colon \mathcal{H}_2^* \to \mathcal{H}_1^*$ such that $\langle F^* u,v \rangle_{\mathcal{H}_1^* \times \mathcal{H}_1} = \langle u,Fv\rangle_{\mathcal{H}_2^*\times\mathcal{H}_2}$ for all $v \in \mathcal{H}_1$, $u \in \mathcal{H}_2^*$. Notice that we identify the spaces $X$ and $Y$ with their dual spaces, so that, e.g., $A^*$ is intended as an operator from $Y$ to $X$. We do not identify $K$ with $K^*$ nor $H$ with $H^*$, so that, e.g., $\iota^*\colon Y \to K^*$.
\begin{proposition}\label{prop:quad_min}
Let
\begin{itemize}
\setlength\itemsep{0em}
    \item $X$, $Y$ and $K$ be separable real Hilbert spaces;
    \item $A\colon X\to Y$ be a bounded map;
    \item $\Se\colon Y\to Y$ satisfy Assumption~\ref{ass:e_rand};
    \item $\iota\colon K\to Y$ be an injective linear map satisfying \eqref{eq:iotaSe};
    \item  $B$ satisfy Assumption~\ref{ass:SigmaAB} and $h\in X$.
\end{itemize}
For $y\in K^*$, the problem 
\begin{equation}
\hat x' = \argmin_{x' \in X} \| \Se^{-1/2} ABx' \|_Y^2 - 2 \lK y-\iota^*Ah, (\Se\iota)^{-1}ABx' \rK + \| x'\|_X^2
\label{eq:quadratic_pb}    
\end{equation}
admits a unique solution, which is
 given by the  bounded affine function $R'_{h,B}\colon K^* \rightarrow X$ defined as
\begin{equation}
R'_{h,B}(y) = (BA^* \Se^{-1} A B + I)^{-1}((\Se\iota)^{-1}AB)^* (y-\iota^*Ah).
    \label{eq:affine2}
\end{equation}
\end{proposition}

\begin{proof}
We have
\[
\hat{x}' = \argmin_{x' \in X} G(x'), \quad G(x') = \| \Se^{-1/2} ABx' \|_Y^2 - 2 \lK y-\iota^* A\vh,(\Se\iota)^{-1}ABx' \rK + \| x'  \|_X^2.
\]
    Since the functional $G$ is strictly convex and differentiable,  we can find its unique minimum by computing the first-order optimality condition. The Gateaux derivative of $G$ in $x'$ along the direction $w$ reads as
\[
\begin{split}
G'(x')[w] &= 2 \langle BA^*\Se^{-1}AB x', w \rangle_X - 2 \lK y-\iota^*A\vh, (\Se\iota)^{-1}AB w \rK + 2 \langle x',w \rangle_X\\
& = 2 \langle BA^*\Se^{-1}AB x', w \rangle_X - 2 \langle ((\Se\iota)^{-1}AB)^* (y-\iota^*A\vh), w \rangle_X + 2 \langle x',w \rangle_X.
\end{split}
\]
 Imposing that $G'(\hat x')[w] = 0$ for all $w \in X$ therefore implies that
\[
(BA^*\Se^{-1}AB + I_X)\hat x' = ((\Se\iota)^{-1}AB)^* (y-\iota^*A\vh).
\]
The operator $BA^*\Se^{-1}AB + I_X$ is invertible since it is a perturbation of the identity by a self-adjoint, non-negative operator, which leads to the expression of the minimizer
\[
\hat{x}' = (BA^*\Se^{-1}AB + I_X)^{-1}(((\Se\iota)^{-1}AB)^* (y-\iota^*A\vh),
\]
 as  in \eqref{eq:affine2}, where we also use that $B$ is self-adjoint since it is positive. 

We now show that $R'_{h,B}\colon K^* \rightarrow X$ is bounded. We need to show that
\begin{equation*}
(BA^*\Se^{-1}AB + I_X)^{-1}(((\Se\iota)^{-1}AB)^*\colon K^*\to X
\end{equation*}
is bounded. As observed above, since $BA^*\Se^{-1}AB$ is self-adjoint and non-negative, we have that $\|(BA^*\Se^{-1}AB + I_X)^{-1}\|_{X\to X}\le 1$. Thus,  it remains to show that
$
((\Se\iota)^{-1}AB)^*\colon K^*\to X
$
is bounded. Recall that this composition is well defined thanks to Assumption~\ref{ass:SigmaAB} (see \eqref{eq:ass23}). We prove that
\begin{equation}\label{eq:comp}
(\Se\iota)^{-1}AB\colon X\to K
\end{equation}
is bounded. By assumption, the map $\Se\iota\colon K\to Y$ is bounded, hence closed. Therefore, $(\Se\iota)^{-1}\colon \iota(K)\subseteq Y\to K$ is closed too. By assumption,  $AB\colon X\to Y$ is bounded, hence closed. Thus, the composition \eqref{eq:comp} is closed, hence bounded thanks to the closed graph theorem.
\end{proof}

In view of this result, the regularized solution $\hat x=h+B\hat x'$ to the inverse problem may be written as in \eqref{eq:affine}:
\begin{equation}
\hat x=R_{h,B}(y)=W_B y + b_{h,B},
    \label{eq:affine3}
\end{equation}
where
\begin{equation}\label{eq:Wandb}
\begin{aligned}
    W_B &=  B(BA^* \Se^{-1} A B + I_X)^{-1} ((\Se\iota)^{-1}AB)^*,\\
    b_{h,B} &= h-B (BA^* \Se^{-1} A B + I_X)^{-1} ((\Se\iota)^{-1}AB)^*\iota^*A\vh,
\end{aligned}
\end{equation}
and $W_B\colon  K^* \rightarrow X$ is bounded.

We now wish to derive an alternative expression for $W_B$. 

\begin{proposition}\label{prop:alternative}
Assume that the hypotheses of Proposition~\ref{prop:quad_min} hold true. The operator $B^2 A^*(\iota^*(AB^2A^*+\Se))^{-1}$ extends to a bounded linear operator from $K^*$ to $X$, which coincides with $W_B$. With an abuse of notation, we have
\begin{equation}\label{eq:WB2}
W_B=B^2 A^*(\iota^*(AB^2A^*+\Se))^{-1}.
\end{equation}
\end{proposition}

\begin{proof}
First, observe that 
  \[
      ((\Se\iota)^{-1}AB)^* \iota^*\Se =
      ((\Se\iota)^{-1}AB)^* (\Se\iota)^*
      = ((\Se\iota)(\Se\iota)^{-1}AB)^* 
      = (AB)^*=BA^*,
  \]
so that
 \begin{equation}\label{eq:composition}
  ((\Se\iota)^{-1}AB)^* \iota^*|_{\Im\Se}=BA^*\Se^{-1}.
   \end{equation}
   This identity will be used below and in the following.
   
In order to prove the result, since the operator $\iota^*(AB^2A^*+\Se)$ is injective, it is enough to show that $W_B$ satisfies
\[
W_B \iota^*(AB^2A^*+\Se)=B^2 A^*.
\]
Replacing the expression of $W_B$ given in \eqref{eq:Wandb}, we obtain
   \[
   B(BA^* \Se^{-1} A B + I_X)^{-1} ((\Se\iota)^{-1}AB)^* \iota^*(AB^2A^*+\Se)=B^2 A^*.
   \]
   Since $B$ satisfies Assumption~\ref{ass:SigmaAB}, we have $\Im (AB^2 A^*)\subseteq\Im(\Se\iota)\subseteq\Im\Se$. Thus, by \eqref{eq:composition}, the above identity is equivalent to
    \[
   B(BA^* \Se^{-1} A B + I_X)^{-1} BA^*\Se^{-1}(AB^2A^*+\Se)=B^2 A^*.
   \]
   In order to prove this identity, it is enough to show that
   \[
   (BA^* \Se^{-1} A B + I_X)^{-1} BA^*\Se^{-1}(AB^2A^*+\Se)=B A^*.
   \]
   Since $B A^*\Se^{-1}AB + I_X$ is invertible with bounded inverse, this identity is equivalent to
   \[
    BA^*\Se^{-1}(AB^2A^*+\Se)=(BA^* \Se^{-1} A B + I_X)B A^*.
   \]
   A quick visual inspection shows that this is always true, concluding the proof.
\end{proof}

\subsection{Proof of Theorem~\ref{prop:quadratic_target_inf2} and of \eqref{eq:minimum}}\label{sec:app3}

The proof of Theorem~\ref{prop:quadratic_target_inf2} is based on the following observation.
\begin{lemma}\label{lem:injectiveB}
Assume that the hypotheses of Theorem~\ref{prop:quadratic_target_inf2} hold true. Let $B_1,B_2\colon X\to X$ satisfy Assumption~\ref{ass:SigmaAB} and suppose that $B_1$ is injective. Then $W_{B_1}=W_{B_2}$ if and only if
\[
B_1^2=B_2^2\quad\text{in }(\ker A)^\perp.
\]
\end{lemma}
\begin{proof}
Using \eqref{eq:Wandb} for $W_{B_1}$ and \eqref{eq:WB2} for $W_{B_2}$, the condition $W_{B_1}=W_{B_2}$ reads
\[
B_1(B_1A^* \Se^{-1} A B_1 + I_X)^{-1} ((\Se\iota)^{-1}AB_1)^* = 
B_2^2 A^*(\iota^*(AB_2^2A^*+\Se))^{-1}.
\]
Since $B_1$ is injective, this is equivalent to
\[
((\Se\iota)^{-1}AB_1)^*\iota^*(AB_2^2A^*+\Se)=(B_1A^* \Se^{-1} A B_1 + I_X)B_1^{-1}B_2^2 A^*.
\]
Since $B_2$ satisfies Assumption~\ref{ass:SigmaAB}, we have $\Im (AB_2^2 A^*)\subseteq\Im(\Se\iota)\subseteq\Im\Se$. Thus, by \eqref{eq:composition},
\[
B_1 A^*\Se^{-1} (AB_2^2A^*+\Se)=(B_1A^* \Se^{-1} A B_1 + I_X)B_1^{-1}B_2^2 A^*.
\]
We readily derive
\[
B_1 A^*\Se^{-1} AB_2^2A^*+B_1 A^*=B_1A^* \Se^{-1} A B_2^2 A^* + B_1^{-1}B_2^2 A^*,
\]
yielding $(B_1^2 -B_2^2) A^*=0$. By continuity of $B_1$ and $B_2$, this is equivalent to having $B_1^2 -B_2^2=0$ in $\overline{\Im A^*}=(\ker A)^\perp$, as desired.
\end{proof}

\begin{proof}[Proof of Theorem~\ref{prop:quadratic_target_inf2}]

Recall that $R_{h,B}$ is given by \eqref{eq:affine3}.

\textit{Step 1: arbitrary affine estimators.}
We first consider the case of an arbitrary affine estimator $y\mapsto Wy+b$, where  $W\colon  K^* \rightarrow X$ is a bounded linear operator and $b \in X$. Thanks to the independence of $x$ and $\varepsilon$ and the fact that $\mathbb{E}\,x=\mu$ and $\mathbb{E}\,\epsilon=0$, the corresponding expected error can be expressed as
\[
\begin{split}
\mathbb{E}_{x,y} [\| Wy +b &- x\|_X^2] =  \mathbb{E}_{x,\epsilon}[\| (W (\iota^* A x +  \varepsilon) +b-x \|_X^2] \\
&=\mathbb{E}_{x,\epsilon}[\| (W \iota^* A - I_X)x + W \varepsilon +b \|_X^2] \\
&=  \mathbb{E}_x[\| (W \iota^* A - I_X)(x-\mu) \|_X^2] + \|  (W \iota^* A - I_X) \mu + b\|_X^2 + \mathbb{E}_\epsilon[\| W \varepsilon \|_X^2] \\
&=  \tr[ (W \iota^* A - I_X) \Sx (W \iota^* A - I_X)^* ] + \tr[W \iota^* \Se \iota W^* ] + \|  (W \iota^* A - I_X) \mu + b\|_X^2,
\end{split}
\]
where the last step  is a consequence of the definition of the covariance operators, e.g.
\[
\begin{split}
\mathbb{E}_\epsilon[\|W\varepsilon\|_X^2] & = \sum_{i} \mathbb{E}_\epsilon[\langle W \varepsilon, \varphi_i \rangle_X^2] \\
&= \sum_{i} \mathbb{E}_\epsilon[\lK \varepsilon, W^*\varphi_i \rK^2]\\
& = \sum_{i} \lY \Se \iota W^*\varphi_i, \iota W^*\varphi_i \rY \\
&= 
\sum_{i} \langle W \iota^* \Se \iota W^*\varphi_i, \varphi_i \rangle_X\\
&= \tr [W \iota^* \Se \iota W^*], 
\end{split}
\]
where $\{\varphi_i\}$ is an orthonormal basis of $X$ and the third identity follows from \eqref{eq:newcov}.

The minimization of the mean square error easily decouples in a minimization in $b$, yielding 
\begin{equation}\label{eq:b}
    b = (I_X-W\iota^* A)\mu,
\end{equation}
and in finding $W$ that minimizes
\[
J(W) = \tr\left[(W\iota^*A-I_X)\Sx(W\iota^*A-I_X)^*+W\iota^*\Se \iota W^*\right].
\]
It is worth observing that, under the introduced hypotheses, such a functional is well-defined. Indeed, since $\iota^*\Se\iota$ is trace-class (cfr.\ eq.\ \eqref{eq:iotaSe}) and $W$ is a bounded operator, the composition $W\iota^*\Se \iota W^*$ defines a trace-class operator, and analogously with the first term, since $\Sx$ is trace-class.

\textit{Step 2: the optimal $B$.}
Let us consider the minimization of $J$. Note that $J$ is convex and differentiable, hence its minimizer can be found by imposing the following first-order optimality condition

Fix an operator $V:K^*\to X$, by imposing that the Gateaux derivative of $J(W)$ along $V$ is zero, we get
\begin{equation}\label{eq:optimalV}
 \tr\left[ V\iota^*(A\Sx A^*+\Se)\iota W^* + W\iota^*(A\Sx A^*+\Se)\iota V^*- V\iota^*A\Sx-\Sx A^*\iota V^* \right] =0.
\end{equation}
Choose $V=w\otimes v:K^*\to X$, where $v\in K$ and $w\in X$, then 
\[ 
\langle \iota^*(A\Sx A^*+\Se)\iota W^* w,v\rangle_{K^*\times K} + \langle W\iota^*(A\Sx A^*+\Se)\iota v,w\rangle_{X} = 
\langle \iota^* A\Sx w,v\rangle_{K^*\times K} + \langle \Sx A^*  \iota v,w\rangle_{X},
\]
so that
\[
\langle W\iota^*(A\Sx A^*+\Se)\iota v,w\rangle_{X} = \langle \Sx A^*  \iota v,w\rangle_{X}.
\]
Since $v$ and $w$ are arbitrary, we get
\begin{equation}\label{eq:optimal}
  W (\iota^* A\Sx A^*\iota + \iota^*\Se \iota ) = \Sx A^* \iota.
\end{equation}
It is easy to show that, if $W$ satisfies~\eqref{eq:optimal}, equality~\eqref{eq:optimalV} holds true for all $V$.
Since $\iota(K)$ is dense in $Y$ and $W$, $A$, $\Sx$ and $\Se$ are bounded,~\eqref{eq:optimal} is also equivalent to
\begin{equation}\label{eq:Wunique}
W \iota^*( A\Sx A^* + \Se  ) = \Sx A^* .
\end{equation}
Observe that the operator $A\Sx A^* + \Se\colon Y\to Y$ is positive and injective, hence it has dense range. Further, $\iota$ is injective, and so $\iota^*\colon Y\to K^*$ has dense range. Thus,  $\iota^*( A\Sx A^* + \Se  )$ has dense range. This shows that there exists at most one bounded operator $W\colon K^*\to X$ satisfying \eqref{eq:Wunique}. Furthermore, Proposition~\ref{prop:alternative} gives that 
$W_{\Sx^{1/2}}$ satisfies \eqref{eq:Wunique}, so that $W_{\Sx^{1/2}}$ is the unique global minimizer of $J$.
Since $\Sx$ is injective, by Lemma~\ref{lem:injectiveB} we have that the $B$'s such that $W_B=W_{\Sx^{1/2}}$ are those satisfying $B^2=\Sx$ in $(\ker A)^\perp$, as desired.

\textit{Step 3: the optimal $h$.}
Let us consider \eqref{eq:b}. It is evident that $b$ is uniquely determined by $W$, and we know that $W=W_{\Sx^{1/2}}$. We now show that,  in the case $b=b_{\Sx^{1/2},h}$ and $W=W_{\Sx^{1/2}}$,  equation \eqref{eq:b} reduces to $h=\mu$. Indeed, we have
\begin{align*}
&&& h-\Sx^{1/2} (\Sx^{1/2}A^* \Se^{-1} A \Sx^{1/2} + I_X)^{-1} ((\Se\iota)^{-1}A\Sx^{1/2})^*\iota^*A\vh = (I_X-W_{\Sx^{1/2}}\iota^* A)\mu\\
&\iff && \Sx^{1/2} (\Sx^{1/2}A^* \Se^{-1} A \Sx^{1/2} + I_X)^{-1} ((\Se\iota)^{-1}A\Sx^{1/2})^*\iota^*A\vh = h-\mu+W_{\Sx^{1/2}}\iota^* A\mu \\
&\iff &&   ((\Se\iota)^{-1}A\Sx^{1/2})^*\iota^*A(\vh-\mu) = (\Sx^{1/2}A^* \Se^{-1} A \Sx^{1/2} + I_X)\Sx^{-1/2}(h-\mu) \\
&\iff &&   \Sx^{1/2} A^*\Se^{-1} A(\vh-\mu) = \Sx^{1/2}A^* \Se^{-1} A (h-\mu) + \Sx^{-1/2}(h-\mu) \\
&\iff &&    \Sx^{-1/2}(h-\mu)=0. \\
\end{align*}
Therefore the optimal value is $\hOp=\mu$.
\end{proof}

\begin{proof}[Proof of~\eqref{eq:minimum}]
We provide an expression for the minimum value of the expected loss, $L(\hOp,\BOp)$. We have
\[
L(\hOp,\BOp) = J(\WOp) = \tr\left[(\WOp \iota^* A-I_X)\Sx(\WOp\iota^*A-I_X)^* +\WOp \iota^* \Se \iota {\WOp}^* \right],
\]
where the optimal linear functional $\WOp$ satisfies \eqref{eq:Wunique}, namely,
$\WOp \iota^* = \Sx A^*(\Se + A\Sx A^*)^{-1}$. 
\[
\begin{aligned}
    J(\WOp) &= \tr\left[(\WOp \iota^*) (A\Sx A^* + \Se) \iota(\WOp)^* - \Sx A^*\iota (\WOp)^* - \WOp\iota^*A\Sx + \Sx \right]\\
    &= \tr\left[ \Sx A^* \iota(\WOp)^* 
    - \Sx A^*\iota (\WOp)^* - \WOp\iota^*A\Sx 
    + \Sx \right] = \\
    & = \tr\left[ 
    \Sx - \big( \Sx^{1/2}(\Sx^{1/2}A^*\Se^{-1}A\Sx^{1/2} + I_X)^{-1} ((\Se\iota)^{-1}A\Sx^{1/2} )^*\iota^*A\Sx  \big)
    \right] \\
    & = \tr\left[ \Sx^{1/2}\left(
    I_X - \big( (\Sx^{1/2}A^*\Se^{-1}A\Sx^{1/2} + I_X)^{-1} \Sx^{1/2} A^*\Se^{-1}A\Sx^{1/2}  \big) \right)\Sx^{1/2}
    \right]\\
    & = \tr\big( \Sx^{1/2}(\Sx^{1/2}A^*\Se^{-1}A\Sx^{1/2} + I_X)^{-1} \Sx^{1/2} \big),
\end{aligned}
\]
where the second line is a consequence of ~\eqref{eq:Wunique} the third line is due to~\eqref{eq:Wandb} and  the forth line holds true by Assumption~\ref{ass:SigmaAB} with $B=\Sx^{1/2}$.
 
 
\end{proof}

\subsection{Proof of Theorem \ref{thm:gener_sup}}
\label{app:proof_supervised}

In order to prove Theorem \ref{thm:gener_sup}, we adapt the classical result on empirical risk minimization to the present discussion, in particular we follow the simple approach in \cite{cucker2002mathematical}. We postpone to a future work the use of more refined techniques 
\cite{bartlett2002rademacher,koltchinskii2011oracle}.
We consider the parameter space $\Theta \subset X \times \mathcal{L}(X,X)$ as in \eqref{eq:Theta} and assume in particular that it satisfies \eqref{eq:2.8new}. We recall that every $B \in \Theta_2$ can be written as $j \spec{B} j^*$, being $\spec{B} \in \operatorname{HS}(H^*,H)$; moreover, since $j = j_2 \circ j_1$, we can also denote it as $ B= j_2 \specc{B} j_2^*$, where $\specc{B}\colon X \to X$, $\specc{B} = j_1 \spec{B} j_1^*$. Notice that, using that  $j_1$ and $j_2$ are injective and have dense range, given $B$, then $\spec{B}$ and $\specc{B}$ are uniquely determined. We can therefore define the following norms on $\Theta$:
    \[
    \begin{aligned}
    \| \theta \|_{*} & = \| (\vh,B)\|_{*} = \max \left\{ \| \vh \|_X, \| \specc{B} \|_{\mathcal{L}(X,X)} \right\}, \\    
    \| \theta \|_{**} & = \| (\vh,B)\|_{**} = \max \left\{ \| \spec{\vh} \|_\vH,  \| \spec{B} \|_{\operatorname{HS}(\vH^*,\vH)} \right\},
    \end{aligned}
    \]
where $\spec{\vh}=j^{-1}(\vh)$. Notice that, according to \eqref{eq:Theta}, the set $\Theta$ can be seen as a closed subset of the ball of radius $\varrho_1$ with respect to $\|\cdot \|_{**}$. Nevertheless, the first result we prove does not require that $\Theta$ is chosen as in \eqref{eq:Theta}, nor that the functional $R_\theta = R_{h,B}$ is as in \eqref{eq:affine}.

The following result is a restatement of Proposition~4 in \cite{cucker2002mathematical}.
\begin{lemma}
Fix  a compact subset $\Theta$ of $X \times \{j_2\specc{B}j_2^*:\specc{B} \colon X\to X\text{ bounded}\}$, endowed with the norm $\| \cdot \|_*$, and a family of functions $R_\theta \colon K^* \rightarrow X$ labelled by $\theta\in \Theta$ satisfying, for a.e.\ $(x,y)\in X \times K^*$:
\begin{enumerate}[a)]
    \item $\| R_\theta(y)-x\|_X \leq M_1 \ $ for every $\theta \in \Theta$;
    \item $\| R_{\theta_1}(y) - R_{\theta_2}(y)\|_X \leq M_2 \| \theta_1 - \theta_2\|_{*}$, for every $\theta_1,\theta_2 \in \Theta$.
\end{enumerate}
Then, with probability $1$ there exist minimizers of  $L$ and $\wh{L}$ over $\Theta$
\[
\thetaOp=\operatornamewithlimits{argmin}_{\theta\in\Theta} L(\theta),  \qquad \thetaS=\operatornamewithlimits{argmin}_{\theta\in\Theta} \wh{L}(\theta), 
\]
and,  for all $\eta > 0$,
\[
\Prob_{\vzv \sim \rho^m} \left[ |L(\thetaS) - L(\thetaOp)| \leq \eta \right] \geq 1 - 2\mathcal{N}\left( \Theta,\frac{\eta}{16M_1 M_2} \right) e^{-\frac{m\eta^2}{8M_1^4}},
\]
where $\mathcal{N}(\Theta,r)$ denotes the \textit{covering number} of $\Theta$, i.e., the minimum number of balls of radius $r$  (in norm $\|\cdot\|_{*}$) whose union contains $\Theta$.
\label{lem:CuckSmale}
\end{lemma}

\begin{proof}
For $\rho$-almost all $(x,y)\in X\times K^*$
\[
\begin{aligned}
\left| \| R_{\theta_1}(y) - x\|_X^2 - \| R_{\theta_2}(y) - x\|_X^2 \right| & = \left| \langle R_{\theta_1}(y) - R_{\theta_2}(y), R_{\theta_1}(y)-x + R_{\theta_2}(y)- x \rangle \right| \\
 & \leq 2M_1 M_2 \| \theta_1 - \theta_2 \|_{*}.
\end{aligned}
\]
By integrating with respect to the probability distribution $\rho$ or the empirical measure $\wh{\rho}$, the above bound holds for  $L,\wh{L}$. Indeed,
\begin{equation}\label{eq:lip1}
\begin{split}
 |L(\theta_1)-L(\theta_2)| &= \left|\E[\|R_{\theta_1}(y) - x\|_X^2] - \E[\|R_{\theta_2}(y) - x\|_X^2]\right| \\
 & \leq \E\left[|\|R_{\theta_1}(y) - x\|_X^2 - \|R_{\theta_2}(y) - x\|_X^2 |\right] \leq 2M_1 M_2 \| \theta_1-\theta_2\|_{*},
 \end{split}
 \end{equation}
 and, with probability $1$,
\begin{equation}\label{eq:lip2}
 \begin{split}
 |\wh{L}(\theta_1)-\wh{L}(\theta_2)| &=  \left| \frac{1}{m} \sum_{j = 1}^m \|R_{\theta_1}(y_j) - x_j\|_X^2 -  \frac{1}{m} \sum_{j = 1}^m \|R_{\theta_2}(y_j) - x_j\|_X^2 \right| \\
 & \leq \frac{1}{m} \sum_{j = 1}^m  \left| \|R_{\theta_1}(y_j) - x_j\|_X^2 - \|R_{\theta_2}(y_j) - x_j\|_X^2 \right| \\
 &\leq 2M_1 M_2 \| \theta_1-\theta_2\|_{*}.
 \end{split}
\end{equation}
Since both $L$ and $\wh{L}$ are Lipschitz continuous and $\Theta$ is compact, the corresponding minimizers $\thetaOp$ and $\thetaS$  exist almost surely.

Next, we  notice that the event $ \{ |L(\thetaS) - L(\thetaOp)| \leq \eta \}$ is a superset of the event $\left\{ \displaystyle \sup_{\theta \in \Theta} |\wh{L}(\theta) - L(\theta)| \leq \eta/2 \right\}$. Indeed, 
\[
\sup_{\theta \in \Theta} |\wh{L}(\theta) - L(\theta)| \leq \frac{\eta}{2} \quad \Rightarrow \quad
L(\thetaS) - \wh{L}(\thetaS) \leq  \frac{\eta}{2} \quad \text{and} \quad \wh{L}(\thetaOp) - L(\thetaOp) \leq \frac{\eta}{2},
\]
and ultimately it also holds that
\[
0\leq L(\thetaS) - L(\thetaOp) =  
\left(L(\thetaS) - \wh{L}(\thetaS)\right)+ 
\left(\wh{L}(\thetaS) - \wh{L}(\thetaOp)\right) + 
\left(\wh{L}(\thetaOp) - L(\thetaOp)\right) \leq \eta,
\]
where we also used the fact that the central difference is negative by  definition of $\thetaS$. Thus,
\[
\Prob_{\vzv \sim \rho^m} \left[ |L(\thetaS) - L(\thetaOp)| \leq \eta \right] \geq \Prob_{\vzv \sim \rho^m} \left[ \sup_{\theta \in \Theta} |\wh{L}(\theta) - L(\theta)| \leq \frac{\eta}{2} \right].
\]

We now provide a lower bound for the latter term. 
In view of \eqref{eq:lip1} and \eqref{eq:lip2}, by using the reverse triangle inequality, for every $\theta_1,\theta_2 \in \Theta$,
\[
 \left| |\wh{L}(\theta_1) - L(\theta_1)| - |\wh{L}(\theta_2) - L(\theta_2)|  \right| \leq 4M_1 M_2 \| \theta_1 - \theta_2 \|_*.
\]
Let now $N=\mathcal{N}\left( \Theta,\frac{\eta}{8M_1 M_2} \right)$ and consider a discrete set $\theta_1, \ldots, \theta_N$ such that the balls $B_k$ centered at $\theta_k$ with radius $\frac{\eta}{8M_1 M_2}$ cover the entire $\Theta$. In each ball $B_k$, for every $\theta \in B_k$ it holds
\[
 \left| |\wh{L}(\theta) - L(\theta)| - |\wh{L}(\theta_k) - L(\theta_k)|  \right| \leq 4M_1 M_2 \| \theta - \theta_k \|_* \leq \frac{\eta}{2}.
\]
Therefore, the event $|\wh{L}(\theta) - L(\theta)| > \eta$ is a subset of $|\wh{L}(\theta_k) - L(\theta_k)| > \frac{\eta}{2}$, and a bound (in probability) of this term can be provided by standard concentration results. Indeed, $\wh{L}(\theta_k)$ is the sample average of $m$ realization of the random variable $\|R_{\theta_k}(y) - x\|_X^2$, whose expectation is $L(\theta_k)$. Moreover, such random variable is bounded by $M_1^2$ by assumption, and therefore via Hoeffding's inequality
\[
\begin{aligned}
 \Prob_{\vzv \sim \rho^m}\left[{\sup_{\theta \in B_k}} |\wh{L}(\theta) - L(\theta)| > \eta \right] & \leq \Prob_{\vzv \sim \rho^m}\left[ |\wh{L}(\theta_k) - L(\theta_k)| > \frac{\eta}{2} \right] \leq 2 e^{-\frac{m\eta^2}{2M_1^4}}.
 \end{aligned}
 \]
 Notice that this inequality holds uniformly in $k$. Finally, since $\Theta$ is covered by the union of the balls $B_1, \ldots, B_N$, with $N = \mathcal{N}\left( \Theta,\frac{\eta}{8M_1 M_2} \right)$, we finally obtain
 \begin{equation*}
 \begin{split}
  \Prob_{\vzv \sim \rho^m} \left[ \sup_{\theta \in \Theta} |\wh{L}(\theta) - L(\theta)| \leq \eta \right]  &= 1 - \Prob_{\vzv \sim \rho^m} \left[ \sup_{\theta \in \Theta} |\wh{L}(\theta) - L(\theta)| > \eta \right] \\
  & \geq
  1- \sum_{k=1}^N \Prob_{\vzv \sim \rho^m} \left[ \sup_{\theta \in B_k} |\wh{L}(\theta) - L(\theta)| > \eta \right] \\
  &\geq 1 - 2N e^{-\frac{m\eta^2}{2M_1^4}}. \qedhere
 \end{split}
 \end{equation*}
\end{proof}

Lemma \ref{lem:CuckSmale} provides a very general result: in order to apply it to our current framework, we have to first show that the functional $R_\theta$ defined as in \eqref{eq:affine} satisfies the assumptions $a)$ and $b)$ in the statement. This is the subject of the following result.

\begin{lemma}
Under the assumptions of Section~\ref{sec:supervised}, let $\Theta$ be as in \eqref{eq:Theta}.
then the family of functions $R_\theta\colon K^* \rightarrow X$, defined by \eqref{eq:affine},  satisfies the assumptions of Lemma \ref{lem:CuckSmale}.
\label{lem:quadratic_reg}
\end{lemma}


\begin{proof} Without loss of generality we assume that $\|j_1\|_{\mathcal{L}(H,X)}\leq 1$ and $\|j_2\|_{\mathcal{L}(X,X)}\leq 1$. 
We first notice that, thanks to \eqref{eq:2.8new}, for any $B_1,B_2 \in \Theta$,
\[
\begin{split}
 \| ((\Se \iota)^{-1}AB_1)^* - ((\Se \iota)^{-1}AB_2)^*\|_{\mathcal{L}(K^*,X)} &\leq  \| (\Se \iota)^{-1}A (B_1-B_2)\|_{\mathcal{L}(X,K)} \\
& = \| (\Se \iota)^{-1}Aj_2 (\specc{B_1}-\specc{B_2})j_2^* \|_{\mathcal{L}(X,K)} 
\\ &\leq \| (\Se \iota)^{-1}Aj_2 \|_{\mathcal{L}(X,K)} \|\specc{B_1}-\specc{B_2}\|_{\mathcal{L}(X,X)}.
\end{split}
\]
Denote by $\varrho_3 = \| (\Se \iota)^{-1}Aj_2 \|_{\mathcal{L}(X,K)}$. Note that, thanks to \eqref{eq:2.8new}, arguing as in the proof of Proposition~\ref{prop:quad_min} we have that $(\Se \iota)^{-1}Aj_2\colon X\to K$ is bounded. Then, by
\eqref{eq:Theta}, for every $B_1,B_2 \in \Theta_2$,
\begin{equation}
\begin{aligned}
\| ((\Se \iota)^{-1}AB_1)^* - ((\Se \iota)^{-1}AB_2)^* \|_{\mathcal{L}(X,K)} &\leq \varrho_3 \|\specc{B_1}-\specc{B_2}\|_{\mathcal{L}(X,X)}, \\
\| ((\Se \iota)^{-1}AB_1)^* \|_{\mathcal{L}(X,K)} &\leq \varrho_1 \varrho_3.
\end{aligned}
    \label{eq:technical_2}
\end{equation}

Assumption \textit{a)} in Lemma \ref{lem:CuckSmale} requires that $\| R_\theta(y) -x \|_X \leq M_1$ for a.e. $(x,y)\in X\times K^*$ and for all $\theta$. Notice that, by the expression of $R_\theta = R_{h,B}$ in \eqref{eq:affine},
\[
\begin{aligned}
\| R_\theta(y)-x\|_X  &\leq \| h \|_X + \| B \|_{\mathcal{L}(X,X)} 
\| (BA^* \Se^{-1} A B + I_X)^{-1}\|_{\mathcal{L}(X,X)} \|((\Se\iota)^{-1}AB)^* \|_{\mathcal{L}(K^*,X)}\\
&\qquad \cdot (\| y \|_{K^*} + \| \iota^*A \|_{\mathcal{L}(X,K^*)} \| h \|_X)+\|x\|_X \\
& \leq \varrho_1 +  \varrho_1^2 \varrho_3 (\varrho_2 + \| \iota^*A \|_{\mathcal{L}(X,K^*)} \varrho_1)+\rho_2 =: M_1 ,
\end{aligned}
\]
where we have also used \eqref{eq:Theta}, \eqref{eq:bound_supp}, \eqref{eq:technical_2} and the fact that  the norm of  $(BA^* \Se^{-1} A B + I_X)^{-1}$ is less than or equal to $1$.

Assumption $b)$ requires instead that $\| R_{\theta_1}(y) - R_{\theta_2}(y) \|_X \leq M_2 \|\theta_1 - \theta_2 \|_*$. According to the definition of $\| \cdot \|_{*}$, we can decouple the perturbation of $\theta$ and study separately the perturbation of $\vh$ and of $B$. We observe that
\[
R_{\vh_1,B}(y) - R_{\vh_2,B}(y) =  (h_1 - h_2) - B(BA^*\Se^{-1} AB+ I_X)^{-1} ((\Se \iota)^{-1}AB)^* \iota^* A(\vh_1 - \vh_2),
\]
hence again by \eqref{eq:technical_2}, \eqref{eq:bound_supp} and \eqref{eq:Theta} we get
\[
\| R_{\vh_1,B}(y) - R_{\vh_2,B}(y) \| \leq (1 +  \varrho_1^2\varrho_3 \| \iota^*A \|_{\mathcal{L}(X,K^*)} ) \| \vh_1 - \vh_2 \|_X. 
\]
The treatment of the perturbations of $B$ is slightly more delicate. Let $C_i = (B_iA^*\Se^{-1}AB_i + I_X)^{-1}$. Then we have
\begin{align*}
&R_{\vh,B_1}(y) - R_{\vh,B_2}(y)\\ 
&\qquad =(B_1 C_1 ((\Se \iota)^{-1}AB_1)^* - B_2 C_2 ((\Se \iota)^{-1}AB_2)^*) (y -\iota^* A h) \\
&\qquad   = (B_1 - B_2) C_1 ((\Se \iota)^{-1}AB_1)^*(y -\iota^* A h) + B_2 (C_1 - C_2) ((\Se \iota)^{-1}AB_1)^*(y -\iota^* A h) \\
&\qquad  \quad + B_2 C_2 (((\Se \iota)^{-1}AB_1)^*-((\Se \iota)^{-1}AB_2)^*)(y -\iota^* A h)
\end{align*}
In the latter summation, by means of \eqref{eq:technical_2}, \eqref{eq:bound_supp} and \eqref{eq:Theta} we easily get that the first and the third terms are both bounded by $\varrho_1 \varrho_3 (\varrho_2 + \| \iota^*A \|_{\mathcal{L}(X,K^*)} \varrho_1) \|\specc{B_1} - \specc{B_2}\|_{\mathcal{L}(X,X)}$. The second term can be reformulated taking into account that
\[
\begin{aligned}
    C_1 - C_2 &= (I_X + B_1 A^*\Se^{-1}AB_1)^{-1} -(I_X + B_2 A^*\Se^{-1}A B_2)^{-1} \\
    &= (I_X + B_1 A^*\Se^{-1}AB_1)^{-1}(B_2 A^*\Se^{-1}A B_2 - B_1 A^*\Se^{-1}A B_1) (I_X + B_2 A^*\Se^{-1}A B_2)^{-1},
\end{aligned}
\]
and its norm can be bounded by  $2 \varrho_1^3 \varrho_3^2 \| \iota^*A \|_{\mathcal{L}(X,K^*)} (\varrho_2 + \| \iota^*A \|_{\mathcal{L}(X,K^*)} \varrho_1) \|\specc{B_1} - \specc{B_2}\|_{\mathcal{L}(X,X)}$ using similar arguments.
\end{proof}

Now that the assumptions $a)$ and $b)$ of Lemma \ref{lem:CuckSmale} are guaranteed, we have to show the compactness of the parameter class $\Theta$. The following lemma only assumes that $\Theta$ is defined as in \eqref{eq:Theta}, by means of a Hilbert space $H$ and a compact, dense-range operator $j\colon H \rightarrow X$.

\begin{lemma}
The set $\Theta$ defined as in \eqref{eq:Theta} is a compact subset of $X \times \{j_2\specc{B}j_2^*:\specc{B} \colon X\to X\text{ bounded}\}$ with respect to the topology induced by the norm $\| \cdot \|_*$.
\label{lem:comp}
\end{lemma}
\begin{proof}
We first show that $\Theta_1\times \Theta_2$ is compact. Set
\begin{equation}\label{eq:thetaspecc}
 \specc{\Theta}_2=\{  j_1 \spec{B} j_1^*\ \colon  \ \spec{B} \in \operatorname{HS}(H^*,H), \| \spec{B} \|_{\operatorname{HS}(H^*,H)} \leq \varrho_1  \}
\end{equation}
so that $\Theta_2=\{j_2 \specc{B} j_2^*\  \colon \ \specc{B}\in\specc{\Theta}_2\}$. The definition of the norm  $\|\cdot\|_*$ implies that $\Theta_1\times \Theta_2$ is compact 
with respect to the topology induced by the norm $\| \cdot \|_*$ if and only if $\Theta_1\times\specc{\Theta}_2$ is compact as subset of $X\times \operatorname{HS}(X,X)$ endowed with the product topology. Hence, it is enough to show that $\Theta_1$ and $\specc{\Theta}_2$ are compact in $X$ and $\operatorname{HS}(X,X)$, respectively. 
By definition, since $j$ is compact and $\Theta_1$ is the image of the closed  ball of radius $\rho_1$ in $\hh$, then $\Theta_1$ is compact. 

%

In order to prove that $\specc{\Theta}_2$ is compact, we  identify $\operatorname{HS}(\hh^*,\hh)$ and $\operatorname{HS}(X,X)$ with $\hh\otimes\hh$ and $X\otimes X$, respectively,
so that for all $v,w\in \hh$, $v\otimes w:\hh^*\to\hh$ is the rank one operator
\[
(v\otimes w)(z)=\langle z,w\rangle_{\hh^*,\hh}\, v, \qquad  z \in H^* .
\]
With this identification, since
\[ j_1( v\otimes w)j_1^*= (j_1 v)\otimes (j_1 w),\]
the map $\spec{B} \mapsto j_1 \spec{B} j_1^*$  is given by 
\[j_1\otimes j_1 : \hh\otimes\hh\to X\otimes X,\]
which is  compact, since $j_1$ is so. As above, $\specc{\Theta}_2$ is the image of the closed  ball of radius $\rho_1$ in $\hh\otimes\hh$, so that it is a compact subset of $\operatorname{HS}(X,X)$.  

The compactness of $\Theta$ follows from the fact that  the subset  of positive operators $\specc{B}\colon X\to X$ is closed in $\operatorname{HS}(X,X)$ and $B=j_2 \specc{B}j_2^*$ is positive if and only if $\specc{B}$ is positive.


\end{proof}

To conclude the proof of Theorem \ref{thm:gener_sup}, we need to provide an explicit expression for the covering numbers of the set $\Theta$ in the $\| \cdot \|_{*}$ norm. This is possible, e.g.\ by assuming the polynomial decay of the singular values of $j_1$ as in \eqref{eq:poly}, by means of some tools that are presented in the next section.

\subsection{Entropy numbers, singular values and covering numbers}
\label{app:entropy}

Let $\mathcal{H}$ and $\mathcal{X}$ be real Hilbert spaces and let $\mathcal{B}$ denote the unit closed ball in $\mathcal{H}$. We use instead the notation $B(v,\eps)$ to denote the closed
ball in $\mathcal{X}$ with center $v$ and radius $\eps$. For any compact operator $T\colon \mathcal{H} \rightarrow \mathcal{X}$ we can define the following quantities.
\begin{enumerate}
\item {\bf Entropy numbers:} for each 
$k\in\N$, $k\geq 1$, 
  \[
 \eps_k(T)= \inf\{\eps>0 \mid \exists v_1,\ldots,v_{k}\in\mathcal X \text{
     such that }
   \cup^k_{i=1} B(v_i,\eps)\supseteq T (\mathcal{B}) \};
\]  
\item {\bf  Singular values:}  $s_k(T)=\lambda_k(|T|)$, where $\lambda_k(|T|)$ is the $k$-th non-zero eigenvalue of $|T|$, which are counted with their multiplicity and ordered in a non-increasing way. If $|T|$ has less than $K$ non-zero eigenvalues, then $s_k(T) =0$ for $k\geq K$.
\item {\bf Covering numbers} of $T$: the covering numbers of the set $T(\mathcal{B})$; namely, for $r>0$, 
\[
\mathcal{N}_r(T) = \mathcal{N}(T(\mathcal{B}),r) = \inf \{ k \in\N_*\mid  \exists v_1,\ldots,v_{k}\in\mathcal{X}\text{ such that } \cup^{k}_{i=1} B(v_i,r)\supseteq T (\mathcal{B})\} .
\]
\end{enumerate}
Properties of covering and entropy numbers have recently been used in the study of  instability in inverse problems \cite{koch2020instability}.
We have the following results (see  \cite{cast90}):
\begin{equation}
\label{eq:4b}
\eps_k(T)\leq r \qquad\Longleftrightarrow\qquad \mathcal N_r(T)\leq k ;
\end{equation}
and for all $k\in\N$, $k\geq 1$
  \begin{equation}
      \sup_{ 1\leq \ell< \infty} \left(  k^{-1/\ell} \,
    \left(\Pi_{i=1}^{\ell}{s_i}(T)\right)^{1/\ell} \right)
  \leq \eps_k(T) \leq 14 \sup_{ 1\leq \ell< \infty} \left(  k^{-1/\ell}\, 
  \left(\Pi_{i=1}^{\ell }{s_i}(T) \right)^{1/\ell} \right).
  \label{eq:5}
\end{equation}

We now use these properties to quantify the covering numbers $\mathcal{N}(\Theta,r)$ appearing in Lemma \ref{lem:CuckSmale}. We assume, for simplicity, that $\varrho_1=1$. For $\varrho_1 \neq 1$, we can rescale the covering numbers by the formula $\mathcal{N}(\varrho B,r) = \mathcal{N}(B,r/\varrho)$, where $\varrho B = \{\varrho b: b \in B\}$. By the definition of $\Theta$, it is evident that $\mathcal{N}(\Theta,r) \leq \mathcal{N}(\Theta_1,r) \mathcal{N}(\Theta_2,r)$. 
 The following two lemmas take care of estimating the covering numbers of $\Theta_1$ and $\Theta_2$, respectively. 
\begin{lemma}
Under Assumption~\eqref{eq:poly} we have
\begin{align}\label{eq:6}
   \ln (\mathcal N(\Theta_1,r)) \leq C r^{-\frac{1}{s}},\qquad r>0,
\end{align}
\label{lem:cover1}
where $C>0$ is independent of $r$. 
\end{lemma}
\begin{proof}
Observe that  $\mathcal{N}(\Theta_1,r) = \mathcal{N}_r(j) \leq  \mathcal{N}_{\frac{r}{\|j_2\|_{\mathcal{L}(X,X)}}}(j_1)$.
Condition \eqref{eq:poly} yields
\[
\Pi_{i=1}^{\ell }s_i(j_1)  \lesssim (\ell!)^{-s} \lesssim \ell^{-s\ell} e^{s\ell},
\]
where the last bound is a consequence of the fact that  $(\ell)!\geq e \ell^{\ell} e^{-\ell}$.
Estimate~\eqref{eq:5} implies that
\[
\eps_k(j_1) \lesssim \sup_{ 1\leq \ell< \infty} \left(  k^{-1/\ell}\, 
  \left( \ell^{-s \ell} e^{s\ell} \right)^{1/\ell} \right) =  \sup_{ 1\leq \ell< \infty} \left(k^{-1/\ell} \ell^{-s}e^s \right).
\]
Let $t=1/\ell$, since the function $e^{-t \ln k}t^{s}$ takes its maximum at $t=s/\ln k$, then
\[
\eps_k(j_1) \lesssim (\ln k)^{-s},
\]
where the constant in $\lesssim$ depends on $s$. Eq.~\eqref{eq:4b} yields $\mathcal{N}_{(\ln k)^{-s}}(j_1) \lesssim k$ and ultimately the thesis.
\end{proof}

\begin{lemma}
Under assumption \eqref{eq:poly}, for every  $s'\in (0,s)$ we have
\begin{align}\label{eq:7}
   \ln (\mathcal N(\Theta_2,r)) \leq C r^{-\frac{1}{s'}},\qquad r>0,
\end{align}
where $C>0$ is independent of $r$.
\label{lem:cover2}
\end{lemma}
\begin{proof}
Observe that $\mathcal{N}(\Theta_2,r) =\mathcal{N}(\specc{\Theta}_2,r)=\mathcal{N}_r(j_1 \otimes j_1)$, 
since $j_1 \otimes j_1$ represents the (compact) embedding of $\operatorname{HS}(H^*,H)$ into $\operatorname{HS}(X,X)$, see the proof of Lemma \ref{lem:comp} and~\eqref{eq:thetaspecc}.

We bound the singular values of $j_1\otimes j_1$.
Let $f\colon(0,+\infty)\to \N$ be defined by
\[
f(t)= \#\{ (k_1,k_2)\in\N_*\times \N_* : s_{k_1,k_2}(j_1 \otimes j_1) = s_{k_1}(j_1) s_{k_2}(j_1)\geq t\}.
\]
Then, for all $k\in \N_*$ we have
\[
s_k (j_1 \otimes j_1) = \sup\{ t \in (0,+\infty): f(t)\geq k\}.
\]
By the polynomial decay of the singular values \eqref{eq:poly} 
\[ f(t) \leq\#\{ (k_1,k_2)\in\N_*\times \N_* :  (k_1 k_2)^{-s}\geq C t\}= g(t),  \]
where $C$ is a suitable constant. Then, 
\[
s_k (j_1 \otimes j_1) \leq \sup\{ t \in (0,+\infty): g(t)\geq k\}.
\]
We now estimate $g(t)$. Let $\tau=(Ct)^{-1/s}$ and $N=[\tau]$ be the integer part of $\tau$. Fix $0<\eps<1$ 
\begin{align*}
    g(t) & =\# \{ (k_1,k_2)\in\N_*\times \N_* :  k_1 k_2 \leq (Ct)^{-1/s}=\tau\}\\
    & = \sum_{k_1=1}^{N} \left[\frac{\tau}{k_1}\right] \leq \sum_{k_1=1}^{N}\frac{\tau}{k_1} \leq \tau + \int_1^{N} \frac{\tau}{x}dx \\
    & \leq \tau (1+ \ln(\tau))\leq \frac{1}{\eps} \tau^{1+\eps} \lesssim \frac{1}{\eps} t^{-(1+\eps)/s}
\end{align*}
 where  $(1+\ln x) \leq \eps^{-1} x^\eps$ provided that $0<\eps<1$. Set $s/2\leq s'=s/(1+\eps)<s$, then 
 \[
 g(t) \lesssim \frac{s}{s-s'} t^{-\frac{1}{s'}}
 \]
so that 
\begin{equation}\label{eq:12}
    s_k(j_1 \otimes j_1) \lesssim  k^{-s'}.
\end{equation}
Clearly, the above bound holds true also when  $0<s'<s/2$. The proof follows by repeating the argument of the proof of Lemma~\ref{lem:cover1}.
\end{proof}

We are now able to prove the main result of section~\ref{sec:supervised}.
\begin{proof}[Proof of Theorem \ref{thm:gener_sup}]
By Lemma \ref{lem:cover1} and \ref{lem:cover2}, we conclude that
\[
\ln(\mathcal N(\Theta,r)) \leq \ln(\mathcal N(\Theta_1,r)) + \ln(\mathcal N(\Theta_2,r))   
\lesssim \ r^{-\frac{1}{s}} +  r^{-\frac{1}{s'}} \lesssim  r^{-\frac{1}{s'}}.
\]
Substituting this result in Lemma \ref{lem:CuckSmale} allows to conclude that
\[
\Prob_{\vzv \sim \rho^m} \left[ |L(\wh{\theta_{\vzv}}) - L(\theta^*)| \leq \eta \right] \geq 1 - e^{\tilde{c}_1 \eta^{-1/s'} - \tilde{c}_2  m\eta^2} = 1- e^{-\tau}.
\]
We can express $\eta$ as a function of $m$ and $\tau$ when $\eta<1$ by the following estimate:
\[
\tau \geq  \tilde{c}_2  m\eta^2 - \tilde{c}_1 \eta^{-1/s'} \quad \Rightarrow \quad \tilde{c}_2 m\eta^{2+\frac{1}{s'}} \leq \tilde{c}_1 + \tau \eta^{1/s'} \leq \tilde{c}_1 +\tau,
\]
and therefore (with constants $c_1,c_2$ independent of $m,\tau,\eta$)
\[
|L(\wh{\theta_{\vzv}}) - L(\theta^*)|  \leq\eta\leq \left( \frac{\tilde{c}_1 + \tau}{\tilde{c_2} m} \right)^{\frac{1}{2+1/s'}} \leq \left( \frac{ c_1 + c_2 \sqrt{\tau}}{\sqrt{m}} \right)^{1 - \frac{1}{2s'+1}},
\]
with probability larger than or equal to $1-e^{-\tau}$.
\end{proof}

\subsection{Proof of Theorem \ref{thm:gener_unsup}} 
\label{app:proof_unsupervised}

In the unsupervised setting, the  regularizer  $R_{\hU,\BU}$ is given by 
\[
R_{\hU,\BU}(y) = \hU+ \BU^2 A^* (\iota^*(A \BU^2 A^*+ \Sigma_\eps))^{-1} (y - \iota^*A\hU),\qquad y\in K^*,
\]
where 
\[
\hU = \hat{\mu} = \frac{1}{m} \sum_{j = 1}^m x_j, \qquad \BU^2= \widehat{\Sx} = \frac{1}{m} \sum_{j=1}^m (x_j - \hat{\mu})\otimes (x_j - \hat{\mu}).
\]
Hence, in order to analyze the statistical properties of $R_{\hU,\BU}$, we first  provide  two concentration inequalities for  $\wh{\mu}$ and $\wh{\Sigma}_x$, which are known since $x$ is a sub-Gaussian random vector in $X$. We include the proofs for the sake of completeness.
\begin{lemma}\label{lem:samp_avg}
Let $x$ be a $\kappa$ sub-Gaussian vector as in~\eqref{eq:subGaussian}. Fix $\tau>0$, then, with probability exceeding $1- 2 e^{-\tau}$, 
  \begin{equation}
    \label{eq:samp_avg}
    \|\widehat{\mu} -\mu\| \leq c\kappa \left( \sqrt{\dfrac{\tr(\Sx)}{m}} + 2\sqrt{\dfrac{ \tau \|\Sx\|}{m}} \right),
  \end{equation}
where $c>0$ is a universal constant.
\end{lemma}
\begin{proof}
Define  $\xi=x-\mu$ where $\mu=\mathbb E[x]$. It is easy to show that $\xi$ is a zero mean $\kappa$ sub-Gaussian vector and its covariance matrix 
is $\Sx = \mathbb{E}[\xi \otimes \xi]$.
The assumption that $\xi$ is a sub-Gaussian random variable \eqref{eq:subGaussian} can be equivalently expressed by requiring that
\[
\| \langle \xi,v\rangle_X \|_{\psi_2} \leq \kappa \| \langle \xi,v\rangle_X \|_2, 
\]
where $\| \langle \xi,v\rangle_X\|_{\psi_2} = \sup_{p\geq 2} \frac{\| \langle \xi,v\rangle_X \|_p}{\sqrt{p}}.$
For each $v$ in the unit ball $B_1$ of $X$, we set $\xi_v=\langle \xi,v\rangle_X$ and we regard $(\xi_v)_{v\in B_1}$ as a random process on $B_1$, viewed as metric space with respect to the metric $\operatorname{d}(v,w) = \| \xi_v - \xi_w \|_2$. Since $B_1=-B_1$, then 
\[
\| \xi \|_X = \sup_{v \in B_1} |\langle \xi,v \rangle| = \sup_{v \in B_1} \langle \xi,v \rangle 
\]
and a standard result of random processes -- see  Exercise~8.6.5 and Theorem~8.5.5 in \cite{vershynin2018high} --  gives that
\[
\Prob \left( \sup_{v \in B_1} |\xi_v| \leq c \kappa (W(B_1)+t\operatorname{diam}(B_1))\right) \geq 1-2 e^{-t^2},\qquad t>0,
\]
where $c$ is a universal constant,  $W$ is the width of the process 
\[
W(B_1) = \E[\sup_{v \in B_1} \langle \xi,v\rangle_X] = \E[\sup_{v \in B_1} |\langle \xi,v\rangle_X|]=\E[\| \xi \|_X],
\]
and $\operatorname{diam}(B_1)$ is the diameter of $B_1$ with respect to the metric $\operatorname{d}(v,w)$
\[
\begin{aligned}
  \operatorname{diam}(B_1) &= \sup_{v,w \in B_1}\operatorname{d}(v,w) = \sup_{v,w \in B_1} \| \langle \xi,v\rangle -  \langle \xi,w\rangle\|_2 = \sup_{v,w \in B_1} \E[\langle \xi,v-w\rangle^2 ]^{1/2} \\
  &=  \sup_{v,w \in B_1}\left( \langle \Sigma_\xi(v-w),v-w \rangle \right)^{1/2} \leq 2 \| \Sigma_\xi \|_{\mathcal{L}(X,X)}^{1/2}.
\end{aligned}
\]
H\"older's inequality implies that 
\[
\E[\| \xi \|_X] \leq \E[\| \xi \|_X^2]^{1/2} = (\tr(\Sx))^{1/2},
\]
so that we get 
\begin{equation}\label{eq:xi}
\Prob \left( \| \xi \|_X \leq c \kappa \left(\sqrt{\tr \Sx }+2t \sqrt{\| \Sx \| }\right) \right) \geq 1-2e^{-t^2}.
\end{equation}

Define $\xi_1=x_1-\mu$, \ldots, $\xi_m=x_m-\mu$, which are  i.i.d. as $\xi$. We claim that there exists an absolute constant $d$ such that $\wh{\xi}= \frac{1}{m} \sum_{j=1}^m\xi_j $ is $d\kappa$ sub-Gaussian. Indeed,
\[
\begin{aligned}
\| \langle \sum_{j  = 1}^m \xi_j,v \rangle \|_{\psi_2}^2 &=  \|  \sum_{j  = 1}^m \langle \xi_j,v \rangle \|_{\psi_2}^2   \\
& \leq d^2 \sum_{j  = 1}^m \| \langle \xi_j,v \rangle \|_{\psi_2}^2 \leq d^2 \kappa^2 \sum_{j  = 1}^m \| \langle \xi_j,v \rangle \|_{2}^2 = d^2 \kappa^2 \| \sum_{j  = 1}^m \langle  \xi_j,v \rangle \|_{2}^2, 
\end{aligned}
\]
where the first inequality is due to the rotational invariance property of sub-Gaussian real random variables  \cite[Proposition~2.6.1]{vershynin2018high}, and the last equality is a consequence of the independence of the variables $\xi_j$. Thus $\wh{\xi} = \wh{\mu} - \mu$ is $d\kappa$ sub-Gaussian. Notice that the covariance of $\wh{\xi}$ is given by
\[
\E[\wh{\xi}\otimes \wh{\xi}] = \frac{1}{m^2} \sum_{i,j} \E[\xi_i \otimes \xi_j] = \frac{1}{m} \sum_{j}  \E[\xi_j \otimes \xi_j] = \frac{1}{m} \Sx.
\]
Setting $\tau = t^2$, by applying \eqref{eq:xi} to $\wh\xi$ we can finally deduce that
\[
\Prob \left( \| \wh{\xi} \|_X \leq cd \kappa\left(\sqrt{\frac{\tr \Sx}{m}}+2\sqrt{\frac{\tau \|\Sx\| }{m}}\right) \right) \geq 1-2e^{-\tau},
\]
which provides the claimed bound by redefining the universal constant $c$. 
\end{proof}
The following lemma is a restatement of a fundamental result in \cite{kolo17}. We include in the statement also the previous inequality. 
\begin{lemma}\label{lem:samp_cov}
Let $x$ be a $\kappa$ sub-Gaussian vector as in \eqref{eq:subGaussian}. Fix $\tau>1$, then, with probability exceeding $1- 3 e^{-\tau}$, 
  \begin{align}
\label{eq:samp_cov}
\| \widehat{\Sx} -\Sx\| & \leq  c\kappa^2 \|\Sx\| \max\left\{ 
\sqrt{\dfrac{\tr\Sx}{m \|\Sx\| }} ,
\dfrac{ \tr\Sx}{m \|\Sx\| }, \sqrt{\dfrac{\tau}{m}}, \dfrac{\tau}{m} \right\}, \\
    \|\widehat{\mu} -\mu\| & \leq c\kappa \left( \sqrt{\dfrac{\tr(\Sx)}{m}} + 2\sqrt{\dfrac{ \tau \|\Sx\|}{m}} \right)\label{eq:samp_avg1},
 \end{align}
where $c$ is a universal constant. 
\end{lemma}
\begin{proof}
We first introduce the operator
\[ 
\widehat{\Sigma}_{\xi} = \frac{1}{m} \sum_{j=1}^m \xi_j\otimes \xi_j = \frac{1}{m} \sum_{j=1}^m  (x_j - \mu) \otimes (x_j - \mu).
\]
Since $\mathbb E[\widehat{\Sigma}_{\xi} ]=\Sx$, Theorem~9 of~\cite{kolo17} gives that
\begin{equation}\label{eq:covariance}
    \| \widehat{\Sigma}_\xi -\Sx\| \leq  c' \|\Sx\| \max\left\{\sqrt{\dfrac{\tr \Sx}{m \|\Sx\| }} ,\dfrac{\tr \Sx}{m \|\Sx\|}, \sqrt{\dfrac{\tau}{m}}, \dfrac{\tau}{m} \right\},
\end{equation}
with probability greater than $1- e^{-\tau}$. As usual, it holds that
\[
\begin{aligned}
\widehat{\Sx} &= \frac{1}{m} \sum_{j=1}^m  (x_j - \widehat{\mu}) \otimes (x_j - \widehat{\mu})  = \frac{1}{m} \sum_{j=1}^m  (x_j - \mu + \mu - \widehat{\mu}) \otimes (x_j - \mu + \mu - \widehat{\mu})\\
& = \frac{1}{m} \sum_{j=1}^m  (x_j - \mu) \otimes (x_j - \mu) +  (\mu - \widehat{\mu}) \otimes \left( \frac{1}{m} \sum_{j=1}^m  (x_j - \mu ) \right) \\
&\quad +\left( \frac{1}{m} \sum_{j=1}^m  (x_j - \mu ) \right)\otimes (\mu - \widehat{\mu}) + (\widehat{\mu} - \mu) \otimes (\widehat{\mu} - \mu) \\
& = \widehat{\Sigma}_{\xi} -  (\widehat{\mu} -\mu) \otimes (\widehat{\mu} -\mu). 
\end{aligned}
\]
As a consequence,
\[
\| \widehat{\Sx} - \Sx \| \leq \| \widehat{\Sigma}_{\xi} - \Sx \| + \| \widehat{\mu} -\mu\|^2.
\]
By~\eqref{eq:covariance} and \eqref{eq:samp_avg}, with probability exceeding $1-3 e^{-\tau}$,  we have both~\eqref{eq:samp_avg1} and 
\begin{align*}
 \| \widehat{\Sx} - \Sx \| &\leq  c' \|\Sx\| \max\left\{ 
\sqrt{\dfrac{\tr \Sx}{m \|\Sx\| }} ,
\dfrac{ \tr \Sx}{m \|\Sx\| }, \sqrt{\dfrac{\tau}{m}}, \dfrac{\tau}{m} \right\}\\
&\quad +  c\kappa^2\left( \dfrac{\tr(\Sx)}{m} + 4\dfrac{\sqrt{\tau \|\Sx\|\tr(\Sx)}}{m} + 4\dfrac{ \tau \|\Sx\|}{m} \right),
\end{align*}
which provides the claimed bounds by redefining the constants $c$. 
\end{proof}

The following lemma shows that the excess risk $L(\hU,\BU) - L(\hOp,\BOp)$ is bounded by $\|\wh{\Sx}-\Sx\|$ and $\|\wh{\mu}-\mu\|$. Note that Lemma~\ref{lem:quadratic_reg} would provide a bound in terms of $\|\wh{\Sx}^{\frac12}-\Sx^{\frac12}\|$. Since the square root is a monotone increasing function, it holds true that
\[
\|\wh{\Sx}^{\frac12}-\Sx^{\frac12}\|\leq \sqrt{\|\wh{\Sx}-\Sx\|},
\]
see Theorem X.1.1 of 
\cite{bhatia2013matrix},
which would provide a worse bound. 
\begin{lemma}
\label{lem:final_unsup}
  Assume that $A\Sigma_x A^* +  \Sigma_\eps \colon Y \to Y$ has a bounded inverse, that the operator 
  \[A^*(\iota^*(A\Sigma_x A^* +  \Sigma_\eps))^{-1}: \iota^*(Y)\subseteq K^*\to X\]
  extends to a bounded operator from $K^*$ to $X$, and that
  \begin{equation}
    \label{eq:31}
   \|(A\Sigma_x A^*+\Sigma_\eps)^{-1} A(\wh{\Sigma}_x-\Sigma_x)A^*\|\leq 1/2.
  \end{equation}
  Then
    \begin{equation}\label{eq:33}
    |L(\hU,\BU) - L(\hOp,\BOp)| = \operatorname{O}\left(\|\wh{\Sx} -\Sx \|\right)+ \operatorname{O}\left(\|\wh{\mu}-\mu\|\right),
\end{equation}
where the constant in $\operatorname{O}$ only depends  on $A,\Sigma_x$, $\Sigma_\eps$, $\iota$ and $\mu$.   
\end{lemma}

\begin{proof}
Let
\[
\xU=R_{\hU,\BU}(y), \qquad \xOp= R_{\hOp,\BOp}(y),
\]
so that
\[
L(\hU,\BU) - L(\hOp,\BOp)= \E\left[ \|  \xU - x \|^2\right] - \E\left[ \| \xOp - x \|^2\right] .
\]
Since $\xOp$ minimizes the mean square error, clearly 
$L(\hU,\BU) - L(\hOp,\BOp)\geq 0$ . We now
prove the upper bound.   Since
\begin{alignat*}{1}
    \|  \xU- x \|^2 - \|\xOp- x \| ^2  & = \|  \xU- \xOp \| ^2 +
    2 \langle \xU- \xOp , \xOp- x \rangle\\
    & \leq \|  \xU- \xOp \| ^2 + 2 \|\xU- \xOp\|  \|  \xOp- x\|,
  \end{alignat*}
then, by H\"older inequality,
\begin{alignat}{1}
      \E\left[ \| \xU- x \|^2\right] - \E\left[ \| \xOp- x
        \|^2\right] & \leq \E\left[ \|\xOp- \xU \|^2\right] +
      2\sqrt{ \E\left[ \|\xU- \xOp \| ^2\right] \E\left[
          \|\xOp-x \|^2\right] } \nonumber \\
      & =\operatorname{O}\left( \sqrt{ \E\left[ \|\xU- \xOp \| ^2\right] }\right),\label{eq:21}
    \end{alignat}
 where the constant in $\operatorname{O}$ only depends on $A,\Sigma_x$ and $\Sigma_\eps$.
 By~\eqref{eq:R2} and the definition of $\xU$
\[
\begin{aligned}
	\xOp &= \Sigma_x A^* (\iota^*(A\Sigma_x A^*+ \Sigma_\eps))^{-1} (y - \iota^*A\mu) +
  \mu = \wt{W} y +\wt{b} \\
\xU & =  \wh{\Sigma}_x A^* (\iota^*(A\wh{\Sigma}_x A^* + \Sigma_\eps))^{-1} (y -
 \iota^* A\wh{\mu}) + \wh{\mu} = \wh{W} y +\wh{b},
	\end{aligned}
      \]
where $\wt{W}$ and $\wh{W}$ are given in \eqref{eq:WB2}, for $B^2=\Sx$ and $B^2 = \wh{\Sx}$, respectively, and $\wt{b}, \wh{b}$ in \eqref{eq:b}.
As a consequence,
    \[
\xU-\xOp = (\wh{W}- \wt{W}) y + \wh{b}- \wt{b} = W y + b  =
W\iota^* A(x-\mu)+W \eps + W \iota^*A\mu + b,
\]
where
\[
  W  = \wh{W}- \wt{W},  \qquad  b   = \wh{b}- \wt{b}.
\]
Hence, 
taking into account that $x-\mu$ and $\eps$ are zero mean random
variables, we obtain
\begin{alignat}{1}
  \E\left[\|\xU- \xOp \| ^2 \right]^{\frac12} & = \left(\tr{\left[ W(\iota^*A\Sigma_x
      A^*\iota+\iota^*\Sigma_\eps \iota)W^* \right]}  + \| W \iota^*A \mu + b\|^2 |\right)
^{\frac12} \nonumber \\
& \leq \tr{\left[ W(\iota^*A\Sigma_x A^*\iota+\iota^*\Sigma_\eps
    \iota)W^* \right]}^{\frac12} + \| W\iota^*A \mu + b\|  \nonumber
\\
&\leq  (\tr{\left( \iota^*A\Sigma_x
    A^*\iota+\iota^*\Sigma_\eps\iota\right)})^{\frac12} \, \|W\| +
\| W\iota^*A \mu + b\| . \label{eq:13}
\end{alignat}
 We now bound  the norm of $W$ where
\[
W= \wh{\Sigma}_x A^* (\iota^*(A\wh{\Sigma}_x
      A^* + \Sigma_\eps))^{-1}-  \Sigma_x A^* (\iota^*(A\Sigma_x A^*+
      \Sigma_\eps))^{-1}.
  \]
A delicate issue is that $(\iota^*(A\wh{\Sigma}_x
      A^* + \Sigma_\eps))^{-1}$ and $(\iota^*(A\Sigma_x A^*+ \Sigma_\eps))^{-1}$
do not have bounded inverses, see  the remark after
Theorem~\ref{prop:quadratic_target_inf2}.  We first prove that  $A\wh{\Sigma}_x
A^*+\Sigma_\eps$ has a bounded inverse. Indeed, let
$\Delta=A(\wh{\Sigma}_x-\Sigma_x)A^*$, then
\begin{align*}
  A\wh{\Sigma}_x
A^*+\Sigma_\eps & = (A\Sigma_x A^*+\Sigma_\eps) \left( I +  (A\Sigma_x
  A^*+\Sigma_\eps)^{-1} \Delta\right).
\end{align*}
By assumption~\eqref{eq:31}, $\|(A\Sigma_x
  A^*+\Sigma_\eps)^{-1} \Delta\|\leq 1/2<1$, so that using Neumann series
  we have that $A\wh{\Sigma}_xA^*+\Sigma_\eps$ is invertible and
    \begin{align*}
&(A\wh{\Sigma}_x A^*+\Sigma_\eps)^{-1} - (A\Sigma_x
      A^*+\Sigma_\eps)^{-1}\\
      &\quad =  \left( I +  (A\Sigma_x
  A^*+\Sigma_\eps)^{-1} \Delta\right)^{-1} (A\Sigma_x A^*+\Sigma_\eps)^{-1} -(A\Sigma_x A^*+\Sigma_\eps)^{-1}\\
   &\quad = \left( \left( I +  (A\Sigma_x
  A^*+\Sigma_\eps)^{-1} \Delta\right)^{-1}- I  \right)    (A\Sigma_x
     A^*+\Sigma_\eps)^{-1} \\
&\quad  = \left( I +  (A\Sigma_x
  A^*+\Sigma_\eps)^{-1}  \Delta\right)^{-1} (A\Sigma_x A^*+\Sigma_\eps)^{-1} \Delta   (A\Sigma_x  A^*+\Sigma_\eps)^{-1}  .     
    \end{align*}
 Then, on  $\iota^*(Y)\subseteq K^*$
    \begin{multline*}
      (i^* (A\wh{\Sigma}_x A^*+\Sigma_\eps))^{-1}  - (i^*(A\Sigma_x
A^*+\Sigma_\eps))^{-1} \\
= \left( I +  (A\Sigma_x
  A^*+\Sigma_\eps)^{-1}  \Delta\right)^{-1} (A\Sigma_x
A^*+\Sigma_\eps)^{-1}A( \wh{\Sigma}_x-\Sigma_x)A^*  (i^* (A\Sigma_x
     A^*+\Sigma_\eps))^{-1}  .
    \end{multline*}
The density of $\iota^*(Y)\subset K^*$ and the assumption that  $A^*  (i^* (A\Sigma_x
     A^*+\Sigma_\eps))^{-1} $ extends to a bounded operator from $K^*$ to $X$ implies
     that
     \begin{align*}
      &\|  (i^* (A\wh{\Sigma}_x A^*+\Sigma_\eps))^{-1}  - (i^*(A\Sigma_x
      A^*+\Sigma_\eps))^{-1}\| \\
&\quad \leq\| \left( I +  (A\Sigma_x
  A^*+\Sigma_\eps)^{-1}  \Delta\right)^{-1}\| \,    
        \| (A\Sigma_x
        A^*+\Sigma_\eps)^{-1}A\|   \| \wh{\Sigma}_x-\Sigma_x\|
        \|   A^*  (i^* (A\Sigma_x A^*+\Sigma_\eps))^{-1}  \| \\
&\quad   \leq  2    \| (A\Sigma_x
        A^*+\Sigma_\eps)^{-1}A\|   \| \wh{\Sigma}_x-\Sigma_x\|
        \|   A^*  (i^* (A\Sigma_x A^*+\Sigma_\eps))^{-1}  \|,
     \end{align*}
     where we used~\eqref{eq:31} to bound $\| \left( I +  (A\Sigma_x
  A^*+\Sigma_\eps)^{-1}  \Delta\right)^{-1}\| $ with $2$, so that
     \begin{equation}
       \label{eq:3}
   \|  (i^* (A\wh{\Sigma}_x A^*+\Sigma_\eps))^{-1}  - (i^*(A\Sigma_x
   A^*+\Sigma_\eps))^{-1}\|     \lesssim  \| \wh{\Sigma}_x-\Sigma_x\|,
     \end{equation}
where the constant in $\lesssim$ only depends on $A,\Sigma_x$ and $\Sigma_\eps$.     
     Since
     \begin{align*}
      W & = \wh{\Sigma}_x A^* \left( (\iota^*(A\wh{\Sigma}_x
      A^* + \Sigma_\eps))^{-1}-  (\iota^*(A\Sigma_x A^*+
      \Sigma_\eps))^{-1} \right) \\
      & \quad + (\wh{\Sigma}_x -\Sigma_x )A^*  (\iota^*(A\Sigma_x A^*+
      \Sigma_\eps))^{-1},
     \end{align*}
eq.~\eqref{eq:3} and the fact that $\| \wh{\Sigma}_x\|\leq
\|\wh{\Sigma}_x-\Sigma_x\|+\|\Sigma_x\|$ both imply 
     \begin{equation}
       \label{eq:4}
       \|W\|\lesssim \| \wh{\Sigma}_x -\Sigma_x\| + \| \wh{\Sigma}_x
       -\Sigma_x\|^2=\operatorname{O}\left( \| \wh{\Sigma}_x -\Sigma_x\|\right),
     \end{equation}     
where the constants in $\lesssim$ and $\operatorname{O}$ only depend on $A,\Sigma_x$ and $\Sigma_\eps$.       
We now observe that
\[
b= (\wh{\mu} - \mu)  - ( \wh{W}\iota^*A \wh{\mu} - \wt{W}\iota^*A \mu) =  (\wh{\mu}
- \mu) - \wh{W}\iota^*A( \wh{\mu} - \mu) - W\iota^*A \mu,
 \]
 so that
 \[
W\iota^*A \mu + b = (I - \wh{W}\iota^*A) (\wh{\mu} - \mu),
\]
and
\[
\| W\iota^*A \mu + b\| \leq \| (I - \wh{W}\iota^*A)\|\,  \|\wh{\mu} - \mu\|\leq
(\| (I - \wt{W}\iota^*A)\| + \|W\iota^*A\|)\, \|\wh{\mu} - \mu\| .
\]
Eq.~\eqref{eq:4} implies that
\begin{equation}
  \label{eq:29}
  \| W\iota^*A \mu + b\| \lesssim  \|\wh{\mu} - \mu\| +  \|\wh{\mu} - \mu\|
  \|\wh{\Sigma}_x-\Sigma_x \|= \operatorname{O}\left(\|\wh{\mu} - \mu\| \right),
\end{equation}
where the constants in $\lesssim$ and $\operatorname{O}$  only depend on $A,\Sigma_x$, $\Sigma_\eps$ and $\mu$.
Eqs.~\eqref{eq:21} and~\eqref{eq:13} with~\eqref{eq:29} give~\eqref{eq:33}.
\end{proof}
We are now able to prove the main result of section~\ref{sec:unsup}.
\begin{proof}[Proof of Theorem \ref{thm:gener_unsup} ]
Since the map $C \mapsto (A\Sigma_x A^*+\Sigma_\eps)^{-1} ACA^*$ is continuous from $\mathcal L(X,X)$ into $\mathcal L(Y,Y)$, there exists $\delta>0$ such that 
\[\|(A\Sigma_x A^*+\Sigma_\eps)^{-1} AC A^*\|\leq 1/2 \qquad \forall C\in \mathcal L(X,X)\quad\|C\|\leq \delta.\]
Set $m_0\in\N$ such that
\[
c \kappa^2\|\Sx\| \max \left\{ 
\sqrt{\dfrac{\tr\Sx}{m_0 \|\Sx\| }} ,
\dfrac{ \tr\Sx}{m_0 \|\Sx\| }, \sqrt{\dfrac{\tau}{m_0}}, \dfrac{\tau}{m_0} \right\} \leq \delta,
\]
where $c$ is the constant in Lemma~\ref{lem:samp_cov}.  
Eq.~\eqref{eq:samp_cov} implies that for all $m\geq m_0$ condition~\eqref{eq:31} is satisfies with probability exceeding $1-4 e^{-\tau}$. 
Possibly redefining $m_0$, by~\eqref{eq:samp_cov} and~\eqref{eq:samp_avg1} we can assume that on the same event
\begin{equation}\label{eq:50}
 \max\{\|\widehat{\mu} -\mu\|, \| \widehat{\Sx} -\Sx\|\}\leq \min\{1,\frac{c_1+c_2\sqrt{\tau}}{\sqrt{m}}\},   
\end{equation}
where $c_1$ and $c_2$ are suitable constants independent of $m$ and $\tau$.
Hence, eq.~\eqref{eq:33} implies that on the same event 
\[
|L(\hU,\BU) - L(\hOp,\BOp)| = \operatorname{O}\left(\|\wh{\Sx} -\Sx \|\right)+  \operatorname{O}\left(\|\wh{\mu}-\mu\|\right)\leq C\frac{c_1+c_2\tau}{\sqrt{m}},
\]
where the last inequality is a consequence of \eqref{eq:50}. Eq.~\eqref{eq:generalization_sup2} is now clear.
\end{proof}
\subsection{Numerical results: further details}
\label{app:numerics}

\subsubsection{Experimental setup}
In Section \ref{sec:numerics}, we set $X=L^2(\mathbb{T}^1)$, being $\mathbb{T}^1$ the one-dimensional torus. For any $N>0$, we can introduce the partition $\{I_{N,i}\}_{i=1}^N$ of the interval $(0,1)$, being $\displaystyle I_{N,i} = \left( \frac{i-1}{N}, \frac{i}{N}\right)$ and define the 1D-pixel basis $\{\varphi_{N,i}\}_{i=1}^N$ as follows:
\[
\varphi_{N,i}(t) = \sqrt{N} \chi_{N,i}(t), \qquad  \chi_{N,i}(t) = \left\{ \begin{aligned} 1 \quad & t \in I_{N,i}, \\ 0 \quad & \text{otherwise.} \end{aligned} \right.
\]
The functions $\{\varphi_{N,i}\}_{i=1}^N$ form an orthogonal set, and we define $X_N$ as the linear space generated by them. Each element $u \in X_N$ can be uniquely represented by a vector $\discr{u} \in \R^N$ as follows:
\[
\discr{u}_i = \frac{1}{|I_{N,i}|} \int_{I_{N,i}} u = \sqrt{N} \langle u, \varphi_{N,i} \rangle_X, \qquad u(t) = \sum_{i=1}^N  \langle u, \varphi_{N,i} \rangle_X \varphi_{N,i}(t) = \sum_{i=1}^N \frac{1}{\sqrt{N}} \discr{u}_i \varphi_{N,i}(t).
\]
As a consequence, for $u \in X_N$, we can compute $\| u \|_X^2 = \sum_{i=1}^N \langle u,\varphi_{N,i}\rangle^2 = \frac{1}{N} \sum_{i=1}^N \discr{u}_i^2$.
The representation of a linear operator $B \colon X_N\rightarrow X_N$ can be done via a matrix $\discr{B} \in \R^{N\times N}$ as follows:
\[
\discr{B}_{i,j} = \langle B \varphi_{N,j}, \varphi_{N,i} \rangle_X;  \qquad v = B u \ \iff \ \discr{v} = \discr{B} \discr{u}.
\]
In order to generate a discrete version of the random process $\varepsilon$ and of the random variable $x$, we first generate the vectors $\discr{\nu}_x$, $\discr{\nu}_\varepsilon$ such that each component $[\discr{\nu}_x]_i$ and $[\discr{\nu}_\varepsilon]_i$ is independently distributed with mean $0$ and covariance $1$. In the proposed tests, we either draw from a Gaussian distribution $\mathcal{N}(0,1)$ or a uniform distribution $\operatorname{Unif}(-\sqrt{3},\sqrt{3})$, taking advantage of the Matlab commands \texttt{randn} and \texttt{rand}. Then in order to approximate the white noise process $\varepsilon$, with zero mean and covariance operator $\Se = \sigma^2 I$, we introduce $\discr{\varepsilon}$ such that
\[
\E [ \discr{\varepsilon}_i \discr{\varepsilon}_j ] = \E [\sqrt{N}\langle \varepsilon ,\varphi_{N,i} \rangle \sqrt{N}\langle \varepsilon ,\varphi_{N,j} \rangle ] = \sigma^2 N \delta_{ij},
\]
thus resulting in $\discr{\varepsilon} = \sigma \sqrt{N} \discr{\nu}_\varepsilon$. As an alternative, we also consider a random process whose components with respect to the Haar wavelet basis are randomly sampled as a white noise, i.e., $\discr{\varepsilon} = \sigma \sqrt{N} \discr{W}^T \discr{\nu}_\varepsilon$, where $\discr{W}$ is the discrete Haar wavelet transform and $\discr{W}^T$ its transpose.

The random variable $\discr{x}$ is instead computed as $\discr{x} = \discr{\mu} + \sqrt{N}\discr{\Sigma}_x^{1/2} \discr{\nu}_x$, being
\[
\discr{\mu}_i = \sqrt{N} \langle \mu,\varphi_{N,i}\rangle_X, \qquad [\discr{\Sigma}^{1/2}_x]_{i,j} =  \langle \Sigma_x^{1/2}\varphi_{N,j}, \varphi_{N,i} \rangle_X.
\]
In the experiments, we picked $\mu(t) = 1-|2t-1|$ and $\Sx^{1/2}$ s.t. 
\[
\Sx^{1/2} u(t) = \int_{\mathbb{T}^1} k_{\Sx}(t')u(t-t')dt, \quad k_{\Sx}(t) = 1-\operatorname{exp}(-(c/t)^4) \chi_{(-c,c)}(t),
\]
being $c=0.2$. Finally, we selected $\sigma = 0.05$. In Figure \ref{fig:3}, we show some signals from the training sample, both in dimension $N=64$ and $N=256$.
\begin{figure}
\centering
\begin{tabular}{ccc}
\includegraphics[width=0.32\columnwidth,trim={2cm 0cm 1cm 0cm}]{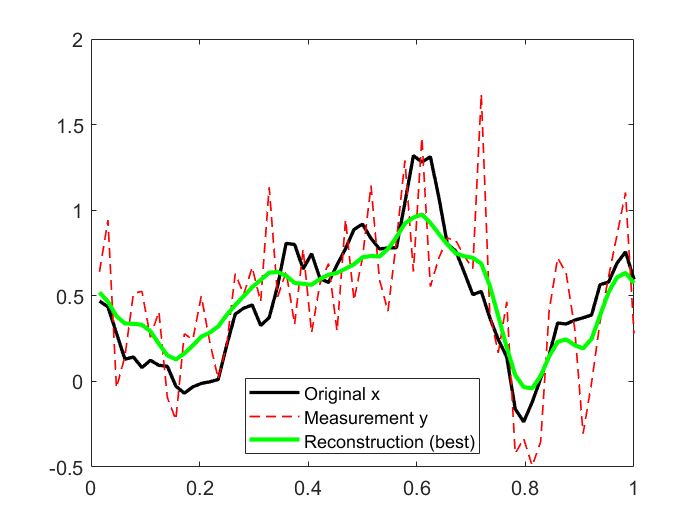}
& \includegraphics[width=0.32\columnwidth,trim={2cm 0cm 1cm 0cm}]{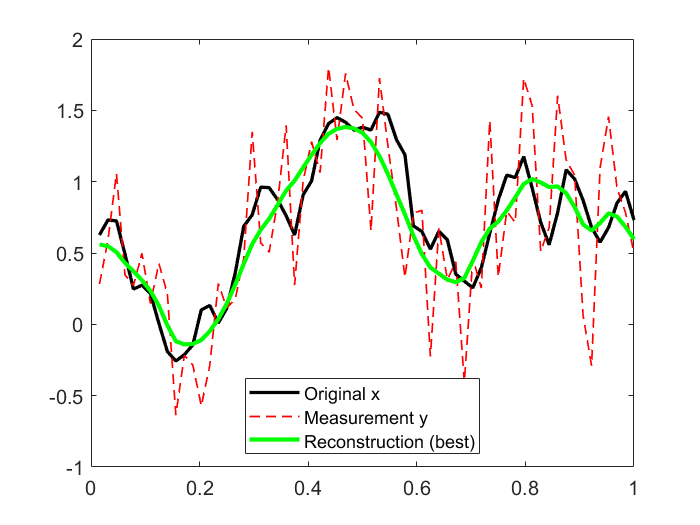} 
& \includegraphics[width=0.32\columnwidth,trim={2cm 0cm 1cm 0cm}]{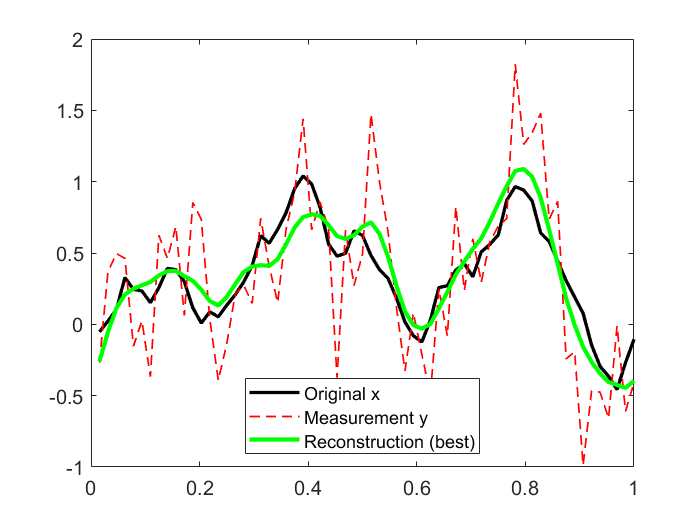} \\
(a) & (b) & (c) \\
\includegraphics[width=0.32\columnwidth,trim={2.0cm 0cm 1.cm 0cm}]{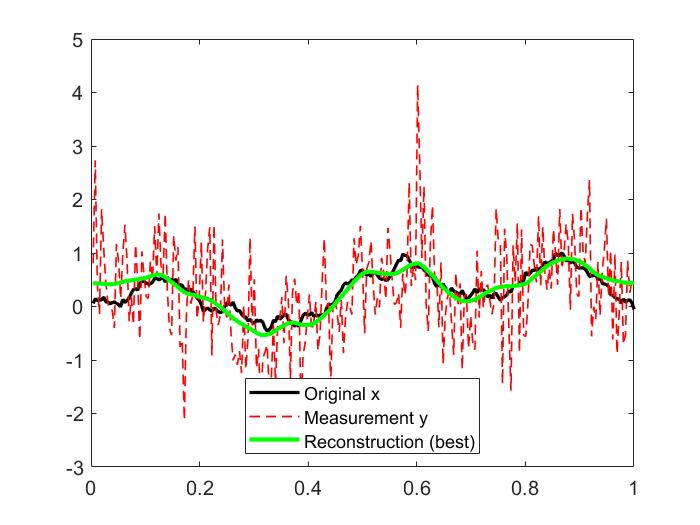}
& \includegraphics[width=0.32\columnwidth,trim={2.0cm 0cm 1.0cm 0cm}]{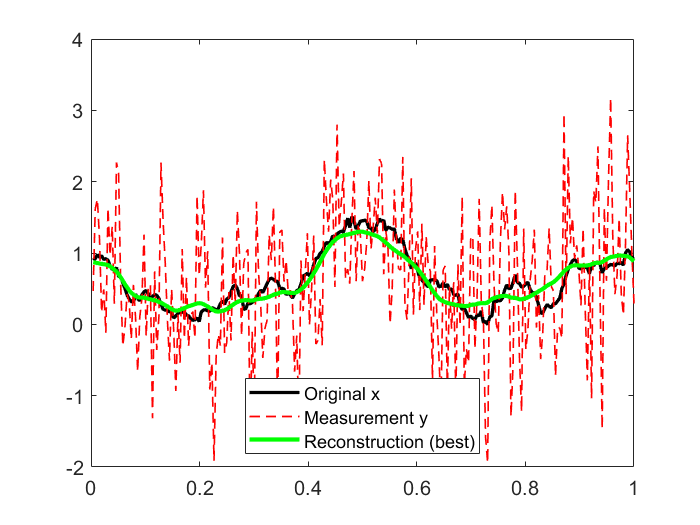} 
& \includegraphics[width=0.32\columnwidth,trim={2.0cm 0cm 1.0cm 0cm}]{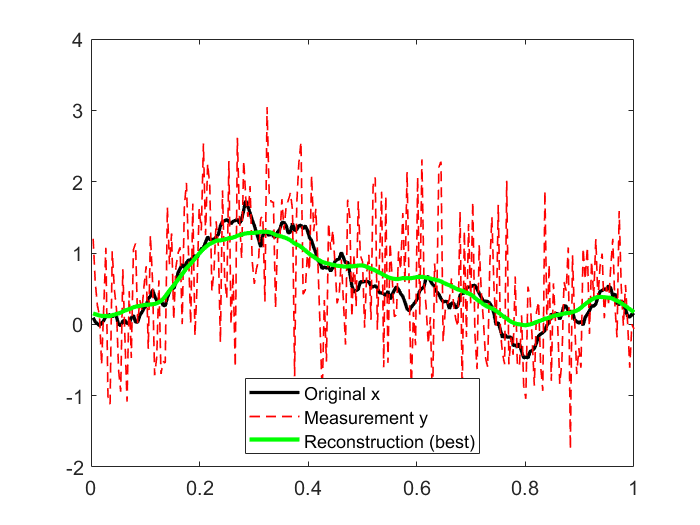} \\
(d) & (e) & (f) 
\end{tabular}
\caption{Signals drawn from the joint distribution in the case where both $x$ and $\varepsilon$ are Gaussian. (a),(b),(c): $N=64$, (d),(e),(f): $N=256$. We show in black the original signal $x$, in red the noisy datum $y = x+\varepsilon$ and in green the reconstruction $R_{\thetaOp}(y)$ associated with the optimal regularizer.
}
\label{fig:3}
\end{figure}

\subsubsection{Implementation aspects}
As expressed in Section \ref{sec:numerics}, it is possible to compute the mean squared error $L$ associated with the optimal parameter $\thetaOp = (\hOp,\BOp)$ and the learned parameters $\thetaS = (\hS, \BS)$ , $\thetaU = (\hU, \BU)$ with an explicit formula. Indeed, since the employed data are synthetically generated, we can take advantage of the knowledge of $\mu, \Sx,\Se$. We simulate the computations in Section \ref{sec:optimal} and in Appendix \ref{sec:app3}, in a finite-dimensional context: here, since for any $N$ the operator $\Se$ is invertible, Assumption \ref{ass:SigmaAB} is satisfied with $K=Y$. The expression of the regularizer in \eqref{eq:affine} then reads as
\[
R_{h,B}(y) = W y + b,
\]
being $W = B^2A^*(\Se + A B^2 A^*)^{-1}$ and $b = (I_X-WA)h$.
Moreover, as in Appendix \ref{sec:app3}, we can compute
\[
L(h,B) = \tr[ (WA - I_X) \Sx (WA - I_X)^* ] + \tr[W \Se W^*] + \|  (WA-I_X) \mu + b\|_X^2.
\]
By this formula, it is possible to compute the mean squared error associated to any parameter $\theta$, and in particular for 
$(\hOp,\BOp) = (\mu,\Sx^{1/2})$ and $(\hU,\BU) = (\wh{\mu},\wh{\Sx}^{1/2})$.
\par

In order to detect the empirical risk minimizer $(\hS,\BS)$, and in particular to compute the quantity $L(\hS,\BS)$, we rely on the same strategy adopted for the minimization of $L$. Therefore, we first look for the affine functional $Wy+b$ which minimizes the empirical risk, defined as
\[
\wh{L}_{b,W} = \frac{1}{m} \sum_{j=1}^m \| Wy_j + b - x_j \|_X^2;
\]
then, if the optimal $b$ and $W$ can be written as $W = B^2 A^*(\Se + A B^2 A^*)^{-1}$ and $b = (I_X-WA)h$, the pair $(h,B)$ is a minimizer of $\wh{L}(h,B)$. Thanks to the empirical mean and covariance matrices
\[
\wh{y} = \frac{1}{m}\sum_{j=1}^m y_j, \qquad \wh{\Sigma}_y = \frac{1}{m} \sum_{j=1}^m (y_j-\wh{y}) \otimes (y_j - \wh{y}), \qquad \wh{\Sigma}_{yx} = \frac{1}{m} \sum_{j=1}^m (y_j-\wh{y}) \otimes (x_j - \wh{\mu}),
\]
it is also possible to provide a more explicit formula for $\wh{L}_{b,W}$. Indeed,
\[   
\begin{aligned}
\wh{L}_{b,W} &= \frac{1}{m} \sum_{j=1}^m \left( \| W(y_j - \wh{y})\|_X^2 + \| x_j-\wh{\mu}\|_X^2 - 2 \langle W(y_j - \wh{y}),x_j - \wh{\mu} \rangle_X + \| W \wh{y} - \wh{\mu} + b \|_X^2 \right) \\
&= \tr[W \wh{\Sigma}_yW^*] + \tr[\wh{\Sigma}_x] -  2\tr[W\wh{\Sigma}_{yx}] + \| W \wh{y} - \wh{\mu} + b \|_X^2,
\end{aligned}
\]
where we have used that $\sum_j (y_j-\wh y)=0$ and $\sum_j (x_j-\wh \mu)=0$.
Thus, the minimizer of $\wh{L}_{b,W}$ is the affine operator associated with $W = \wh{\Sigma}_{xy}\wh{\Sigma}_{y}^{-1}$ and $b = \wh{\mu}- W\wh{y}$. Unfortunately, such $W$ does not yield the optimal parameter $\BS$: indeed, $W$ cannot be written in the form $W = B^2 A^*(\Se + A B^2 A^*)^{-1}$, but rather $W = M A^*(\Se + AM A^*)^{-1}$, where the resulting $M$
is not symmetric. We overcome such issue by considering the symmetric part of $M$, which we denote by $M'$. Indeed, despite the operator $W'$ associated with $M'$ is possibly different from the minimizer of $\wh{L}_{b,W}$ among the functionals of the form $W = B^2 A^*(\Se + AB^2 A^*)^{-1}$, numerical evidence shows that the values of $L$ evaluated in $W'$ and $W$ are very close
. Since the former is an upper bound of $L(\hS,\BS)$ and the latter a lower bound, we conclude that the expected loss $L$ evaluated in $W'$ provides a sufficiently tight upper estimate of the value of $L(\hS,\BS)$, without explicitly requiring the computation of $\BS$ and $\hS$.

As a final remark, we show how the generalization bounds in probability obtained in Theorems \ref{thm:gener_sup} and \ref{thm:gener_unsup} can be reformulated in expectation. Let us first consider the unsupervised case, in which 
\[
\Prob_{\vzv \sim \rho^m} \left[ |L(\hU,\BU) - L(\hOp,\BOp)| \leq c_3 \frac{1}{\sqrt{m}} + c_4 \frac{\sqrt{\tau}}{\sqrt{m}} \right] \geq 1 - e^{-\tau}.
\]
Inverting $\eta = c_3 \frac{1}{\sqrt{m}} + c_4 \frac{\sqrt{\tau}}{\sqrt{m}}$ in terms of $\tau$ we get
\[
\Prob_{\vzv \sim \rho^m} \left[ |L(\hU,\BU) - L(\hOp,\BOp)| \leq \eta \right] \geq 1 - \wt{c_1}e^{-\wt{c_2}m\eta^2}.
\]
This can be translated into a bound in expectation by means of the following identity:
\[
\mathbb{E}_{\vzv \sim \rho^m} \left[ |L(\hU,\BU) - L(\hOp,\BOp)|\right] = \int_0^\infty \Prob_{\vzv \sim \rho^m} \left[ |L(\hU,\BU) - L(\hOp,\BOp)| > \eta \right]d\eta 
\lesssim \frac{1}{\sqrt{m}}.
\]

In a similar way, in the supervised case we have (see the Appendix \ref{app:entropy})
\[
\Prob_{\vzv \sim \rho^m} \left[ |L(\wh{\theta_{\vzv}}) - L(\theta^*)| \leq \eta \right] \geq 1 - e^{c_1 \eta^{-1/s'} - c_2  m\eta^2}.
\]
Notice that, when $\eta \rightarrow 0$, such bound could be meaningless, as the term $e^{c_1 \eta^{-1/s'}}$ blows up. We therefore substitute it with the following estimate:
\[
\Prob_{\vzv \sim \rho^m} \left[ |L(\wh{\theta_{\vzv}}) - L(\theta^*)| \leq \eta \right] \geq 1 - \min\{1,e^{c_1 \eta^{-1/s'} - c_2  m\eta^2}\}.
\]
As a consequence, 
\[
\begin{aligned}
\mathbb{E}_{\vzv \sim \rho^m} \left[ |L(\hU,\BU) - L(\hOp,\BOp)|\right] &= \int_0^\infty \Prob_{\vzv \sim \rho^m} \left[ |L(\hU,\BU) - L(\hOp,\BOp)| > \eta \right]d\eta \\
&\leq \int_0^\infty \min\{1,e^{c_1 \eta^{-1/s'} - c_2  m\eta^2}\} d\eta.
\end{aligned}
\]
Notice that $1 \leq e^{c_1 \eta^{-1/s'} - c_2  m\eta^2}$ when $c_1 \eta^{-1/s'} \geq c_2m \eta^2$, namely when $\eta \leq \wh{\eta}(m)=\left(\frac{c_1}{c_2 m}\right)^{\frac{1}{2+1/s'}}$. Thus,
\[
\begin{aligned}
\mathbb{E}_{\vzv \sim \rho^m} &\left[ |L(\hU,\BU) - L(\hOp,\BOp)|\right]  \leq  \wh{\eta}(m) + \int_{\wh{\eta}(m)}^\infty e^{c_1 \eta^{-1/s'} - c_2  m\eta^2} d\eta \\
& \qquad \leq  \wh{\eta}(m) + \int_{\wh{\eta}}^\infty e^{c_1 \wh{\eta}^{-1/s'} - c_2  m\eta^2} d\eta  = \left[\begin{aligned}
& \quad c_1 \wh{\eta}^{-1/s'} - c_2  m\eta^2 = -\beta^2 \\
&d\eta = \frac{1}{\sqrt{c_2 m}}\frac{\beta}{\sqrt{c_1 \wh{\eta}^{-1/s'} + \beta^2}}d\beta\leq\frac{1}{\sqrt{c_2 m}}d\beta
\end{aligned} \right] \\
& \qquad \leq  \wh{\eta}(m) + \frac{1}{\sqrt{c_2 m}} \int_{0}^\infty e^{-\beta^2} d\beta  \lesssim \left(\frac{1}{m}\right)^{\frac{1}{2+1/s'}} + \frac{1}{\sqrt{m}},
\end{aligned}
\]
 and the leading order is $\left(\frac{1}{m}\right)^{\frac{1}{2+1/s'}}$, which can be rewritten as $\left(\frac{1}{\sqrt{m}}\right)^{1-\frac{1}{2s'+1}}$, and converges to $\frac{1}{\sqrt{m}}$ for large values of $s'$ (namely, of $s$).

Finally, in Table \ref{tab:1} we report the numerical values of the excess risk $|L(\hU,\BU)-L(\hOp,\BOp)|$ and $|L(\hS,\BS)-L(\hOp,\BOp)|$ associated with all the studied cases.
\begin{table}%
\begin{tabular}{l|ccccccc}
& \multicolumn{7}{c}{Sample size, $m$} \\
\multicolumn{1}{c|}{Model (a)}  & $3000$  & $6463$ & $13925$ & $30000$ & $64633$ & $139248$ & $300000$\\
\hline
$N=64$, unsup. &  $0.00174$  &  $0.00119$  &  $0.00076$  &  $0.00056$  &  $0.00038$  &  $0.00025$  &  $0.00018$ \\
$N=64$,  sup.  &  $0.00398$  &  $0.00233$  &  $0.00156$  &  $0.00111$  &  $0.00067$  &  $0.00044$  &  $0.00031$ \\
$N=256$, unsup. & $0.00195$  &  $0.00131$  &  $0.00086$  &  $0.00063$  &  $0.00040$  &  $0.00029$  &  $0.00020$ \\ 
$N=256$,  sup. &  $0.01369$  &  $0.00485$  &  $0.00246$  &  $0.00134$  &  $0.00092$  &  $0.00052$  &  $0.00037$ 
\end{tabular}
\\
\begin{tabular}{l|ccccccc}

\multicolumn{1}{c|}{Model (b)}& & & & & & & \\
\hline
$N=64$, unsup. &  $0.00177$  &  $0.00129$  &  $0.00082$  &  $0.00053$  &  $0.00034$  &  $0.00023$  &  $0.00019$ \\
$N=64$,   sup. &  $0.00380$  &  $0.00236$  &  $0.00158$  &  $0.00103$  &  $0.00058$  &  $0.00044$  &  $0.00032$ \\
$N=256$, unsup. & $0.00199$  &  $0.00126$  &  $0.00094$  &  $0.00056$  &  $0.00044$  &  $0.00028$  &  $0.00019$ \\ 
$N=256$,  sup. &  $0.01449$  &  $0.00487$  &  $0.00250$  &  $0.00142$  &  $0.00086$  &  $0.00058$  &  $0.00035$
\end{tabular}
\begin{tabular}{l|ccccccc}
\multicolumn{1}{c|}{Model (c)} & & & & & & & \\
\hline
$N=64$,  unsup.&  $0.00188$  &  $0.00125$  &  $0.00083$  &  $0.00056$  &  $0.00039$  &  $0.00027$  &  $0.00018$ \\
$N=64$,    sup. &  $0.00407$  &  $0.00240$  &  $0.00154$  &  $0.00105$  &  $0.00069$  &  $0.00044$  &  $0.00028$ \\
$N=256$, unsup.&  $0.00190$  &  $0.00134$  &  $0.00088$  &  $0.00057$  &  $0.00040$  &  $0.00028$  &  $0.00019$ \\
$N=256$,   sup. &  $0.01434$  &  $0.00503$  &  $0.00248$  &  $0.00135$  &  $0.00089$  &  $0.00052$  &  $0.00036$
\end{tabular}
\vspace{5mm}
\caption{Tabulated values of the excess risks associated with Figure \ref{fig:2}, computed at two discretization levels and in three different statistical setups: Gaussian variable $x$ and (a) uniform white noise $\varepsilon$, (b) Gaussian white noise $\varepsilon$, and (c) white noise $\varepsilon$ uniformly distributed w.r.t. the Haar wavelet transform.}
\label{tab:1}
\end{table}

\subsection{An ill-posed inverse problem: deconvolution of 1D signals}
\label{app:deconv}
We provide a numerical verification of the estimates of Theorems~\ref{thm:gener_sup} and \ref{thm:gener_unsup} for a 1D deconvolution problem, extending the experiments of section \ref{sec:numerics} to the case of an ill-posed operator $A$. We consider again $X = Y = L^2(\T^1)$, and introduce the convolution operator $(Ax)(t) = (k \ast x)(t) = \int_{\T} k(t -\tau) x(\tau) d\tau$. This operation can be used to describe the blurring of one-dimensional signals, the function $k$ being the convolutional filter, or point spread function. In our experiments, we consider $k(t) = \upchi_{[-L,L]}(t)$, the indicator function of the interval $[-L,L]$, and set $L=0.02$. Such a kernel $k$ can be referred to as the average filter. When discretizing the interval $\T^1$ with $N$ 1D-pixels, the operator $A$ reduces to a discrete (periodic) convolution with a constant vector $\mathbf{k}$, whose number of entries is $LN$. Both at a continuous and at a discrete level, the deconvolution problem is known to be ill-posed (see, e.g., \cite{mueller2012linear}), and the smallest singular value of the discretized operator vanishes as $N$ grows. Nevertheless, we expect to observe the same generalization bounds as in the denoising case.
\par
We replicate the same experiments as in section \ref{sec:numerics}, assuming that $x$ is a random Gaussian variable (with mean $\mu$ and covariance $\Sx$ as reported in section \ref{sec:numerics}) and $\varepsilon$ is white uniform noise with covariance $\Se = \sigma^2 I$. We fix a noise level of $2.5\%$ by setting $\sigma$ equal to the $2.5\%$ of the peak value of the average signal. The results of the numerical experiments are reported in Figure \ref{fig:4}. We observe that in both scenarios the decay of the excess risk is of the order $1/\sqrt{m}$, and the unsupervised technique still provides (slightly) better results, which in particular are not affected by the increased ill-posedness of the operator at a refined scale.
\begin{figure}
\centering
\begin{tabular}{ccc}
\includegraphics[width=0.32\columnwidth,trim={1cm 0cm 1cm 0cm}]{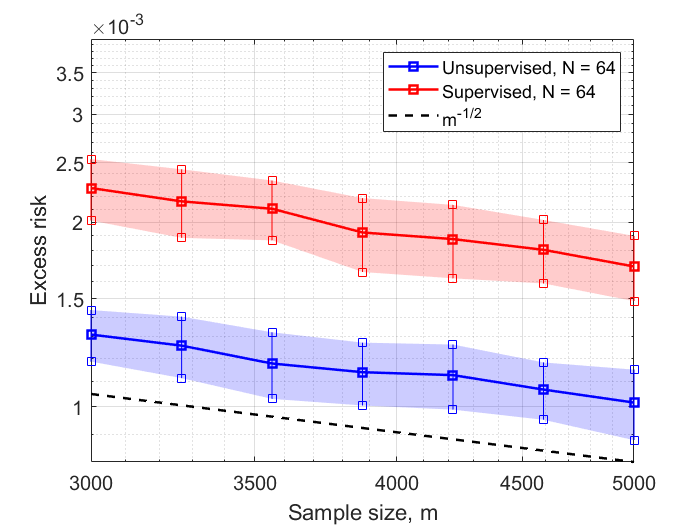}
& \includegraphics[width=0.32\columnwidth,trim={1cm 0cm 1cm 0cm}]{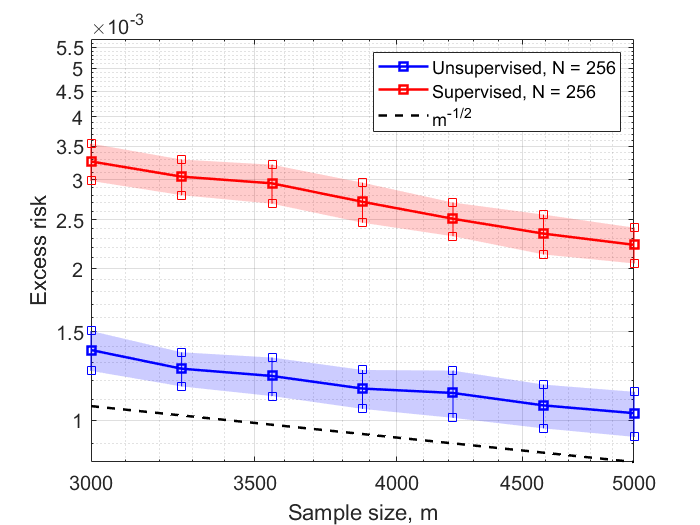} 
& \includegraphics[width=0.32\columnwidth,trim={1cm 0cm 1cm 0cm}]{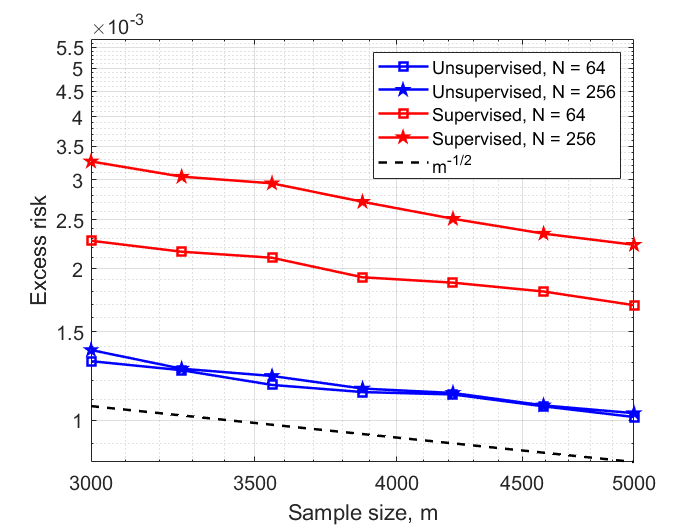} \\
(a) & (b) & (c)
\end{tabular}
\caption{Decay of the excess risks \textcolor{red}{$|L(\thetaS)-L(\thetaOp)|$} and \textcolor{blue}{$|L(\thetaU)-L(\thetaOp)|$} (with standard deviation error bars) with two different discretization sizes, $N=64$ (a) and $N=256$ (b), and comparison (c).}%
\label{fig:4}%
\end{figure}



\end{document}